\documentclass[a4paper]{article}

\usepackage[a4paper,top=3cm,bottom=2cm,left=3cm,right=3cm,marginparwidth=1.75cm]{geometry}\usepackage{booktabs} % For formal tables
\usepackage[utf8]{inputenc}
\usepackage{bm}
\usepackage{amsmath}
\usepackage{mathtools}
\usepackage{amssymb}
\usepackage{bm} 
\usepackage{amsthm}
\usepackage{graphicx}
\usepackage{xcolor}
\usepackage{optidef}
\usepackage{dsfont}
\usepackage{upgreek}
\usepackage[ruled]{algorithm2e} % For algorithms

\SetAlFnt{\small}
\SetAlCapFnt{\small}
\SetAlCapNameFnt{\small}
\SetAlCapHSkip{0pt}
\IncMargin{-\parindent}
\newcommand{\bx}{\mathbf{x}}

\newcommand{\bu}{\mathbf{u}}

\newcommand{\by}{\mathbf{y}}

\newcommand{\bz}{\mathbf{z}}
\newcommand{\bv}{\mathbf{v}}

\usepackage[colorinlistoftodos]{todonotes}
\newtheorem{theorem}{Theorem}%[section]
\newtheorem{proposition}{Proposition}%[section]
\newtheorem{lemma}{Lemma}%[section]
\newtheorem{corollary}{Corollary}%[theorem]
\newtheorem{definition}{Definition}

\newcommand\ignore[1]{}
\newcommand{\lh}{\textcolor{red}}
\DeclarePairedDelimiter\abs{\lvert}{\rvert}%
\DeclarePairedDelimiter\norm{\lVert}{\rVert}%
\DeclareMathOperator*{\argmax}{arg\,max}
\DeclareMathOperator*{\argmin}{arg\,min}
\newcommand{\defeq}{\vcentcolon=}

\SetCommentSty{mycommfont}

\makeatletter
\let\oldabs\abs
\def\abs{\@ifstar{\oldabs}{\oldabs*}}
\let\oldnorm\norm
\def\norm{\@ifstar{\oldnorm}{\oldnorm*}}

\begin{document}

\title{Fair Classification and Social Welfare}
\author{Lily Hu and Yiling Chen}

\maketitle

\begin{abstract}
%Many concerns about fair algorithmic systems revolve around the distribution of important social resources: credit, employment opportunities, second-chances at public life. 
Now that machine learning algorithms lie at the center of many important resource allocation pipelines, computer scientists have been unwittingly cast as partial social planners. Given this state of affairs, important questions follow. What is the relationship between fairness as defined by computer scientists and longer standing notions of social welfare? In this paper, we present a welfare-based analysis of classification and fairness regimes. We translate a loss minimization program into a social welfare maximization problem with a set of implied welfare weights on individuals and groups---weights that can then be analyzed from a distribution justice lens. Working in the converse direction, we ask what the space of possible labelings is for a given dataset $\mathcal{X}$ and hypothesis class $\mathcal{H}$. We provide an algorithm that answers this question with respect to linear hyperplanes in $\mathbb{R}^d$ that runs in $O(n^dd)$. Our main findings on the relationship between fairness criteria and welfare center on sensitivity analyses of fairness-constrained empirical risk minimization programs. We characterize the ranges of $\Delta \epsilon$ perturbations to a fairness parameter $\epsilon$ that yield better, worse, and neutral outcomes in utility for individuals and by extension, groups. We show that applying more strict fairness criteria that are codified as parity constraints, can worsen welfare outcomes for both groups. More generally, always preferring ``more fair'' classifiers does not abide by the Pareto Principle---a fundamental axiom of social choice theory and welfare economics. Recent work in machine learning has rallied around these notions of fairness as critical to ensuring that algorithmic systems do not have disparate negative impact on disadvantaged social groups. By showing that these constraints often fail to translate into improved outcomes for these groups, we cast doubt on their effectiveness as a means to ensure fairness and justice.  
\end{abstract}

\section{Introduction}
In his 1979 Tanner Lectures, Amartya Sen noted that since nearly all theories of fairness are founded on an equality of \textit{some} sort, the heart of the issue rests on clarifying the ``equality of what?'' problem \cite{sen}. The field of fair machine learning has not escaped this essential question. Does machine learning have an obligation to assure probabilistic equality of outcomes across various social groups \cite{feldman2015certifying,hardt2016equality}? Or does it simply owe an equality of treatment \cite{dwork2012fairness}? Does fairness demand that individuals (or groups) be subject to equal mistreatment rates \cite{zafar2017fairness,bechavod2017learning}? Or does being fair refer only to avoiding some intolerable level of algorithmic error? 

Currently, the task of accounting for fair machine learning cashes out in the comparison of myriad metrics---probability distributions, error likelihoods, classification rates---sliced up every way possible to reveal the range of inequalities that may arise before, during, and after the learning process. But as shown in Chouldechova \cite{chouldechova2017fair} and Kleinberg et al. \cite{kleinberg2017inherent}, fundamental statistical incompatibilities rule out any solution that can satisfy all parity metrics. Fairness-constrained loss minimization offers little guidance on its own for choosing among the fairness desiderata, which appear incommensurable and result in different impacts on different individuals and groups. We are thus left with the harsh but unavoidable task of adjudicating between these measures and methods. How ought we decide? For a given application, who actually benefits from the operationalization of a certain fairness constraint? This is a basic but critical question that must be answered if we are to understand the impact that fairness constraints have on classification outcomes. Much research in fairness has been motivated by the well-documented negative impacts that these systems can have on already structural disadvantaged groups. But do fairness constraints as currently formulated in fact earn their reputation as serving to improve the welfares of marginalized social groups? 

As algorithms continue to be adopted in social environments---consider, for example, the place of predictive systems in the financial services industry---classifier performance and outcomes directly bear on individuals' welfares. In light of this new terrain of machine learning, our paper views predictions as \textit{resource allocations} awarded to individuals and by extension, to various social groups. With this orientation in mind, we build out a conceptual framework and methodology that analyzes classifications and fairness regimes from a utility and welfare-centric perspective.

In Section 3, we cast the loss minimization task at the center of supervised learning as a social welfare maximization problem prevalent in social choice theory and welfare economics. In the Planner's Problem, a social planner seeks to maximize social welfare represented as the sum of weighted utility functions, where each individual's weight represents the value placed by society on her welfare. Inverting the Planner's Problem of efficient social welfare maximization generates a question that concerns social equity: \textit{``Given a particular allocation, what is the presumptive social weight function that would yield the allocation as optimal?''} We show that the set of predictions issued by the optimal classifier of any loss minimization task can also be given as the set of optimal allocations in the Planner's Problem, over the same individuals, endowed with a given set of welfare weights. These weights lie at the heart of debates over fairness of distribution in economics.  

A demonstration of the converse result---given a social welfare maximizing allocation, what is the hypothesis that can achieve an equivalent classification?---depends on the particulars of a given dataset and the hypothesis class under a learner's consideration. In Section 4, we provide an algorithm that computes all achievable labelings in $O(n^dd)$ time for the class of linear $d$-dimensional hyperplanes. Because the fair machine learning literature focuses on classification outcomes across different protected attribute social groups \textit{e. g.}, race, this approach exhaustively records the set of utilities---defined by the number of positively labeled individuals belonging to that group---achievable for various social groups. 
%Since each point on this utility simplex is supported by a particular classifier, it is also associated with a learner loss as well as a set of ``fairness scores'' that relay the classifier's performance according to given statistical fairness definitions. 

Our main result is presented in Section 5 and shows that how ``fair'' a classifier is---how well it accords with a group parity constraint such as equality of opportunity or balance for false positives---does not neatly translate into statements about how it impacts different groups' welfares. Using techniques from parametric programming and finding a SVM's regularization path, we show that so long as a fairness constraint binds, \textit{i.e.}, applying the constraint changes the optimal SVM solution, tightening the $\epsilon$-level fairness constraint always leads to learner loss but does not necessarily improve classification outcomes for either group. In particular, we prove two surprising results: first, starting at any nonzero $\epsilon$-fair optimal SVM solution, there exists a $\Delta \epsilon < 0$ perturbation that tightens the fairness constraint and leads to classifier-output allocations that are weakly Pareto dominated by those issued by the ``less fair'' original classifier. Second, there are nonzero $\epsilon$-fair optimal SVM solutions, such that there exist $\Delta \epsilon < 0$ perturbations that yield classifications that are strongly Pareto dominated by those issued by the ``less fair'' original classifier. We demonstrate these findings on the Adult dataset. In general, our results show that when notions of fairness rest entirely on leading parity-based notions, always preferring more fair machine learning classifiers does not accord with the Pareto Principle, an axiom typically seen as fundamental in social choice theory and welfare economics generally. 

The purposes of our paper are twofold. The first is simply to encourage a more welfare-centric understanding of algorithmic fairness. Whenever machine learning is deployed within important social and economic processes, concerns for fairness arise when shared societal norms are in tension with the decision-maker's goals and desires. Most leading methodologies have focused on optimization of utility or welfare to the vendor, limiting our ability to answer questions about how individuals, groups, and society-at-large fare under various distributive allocations. The social welfare perspective directly engages both questions of efficiency, in the task of maximization, and equity, in the design of welfare weights. This perspective is especially enlightening when applied to sectors in which the government, acting as the Planner, maintains a strong interest in issues of distributive fairness and can justifiably make interpersonal comparisons of utility. 

We also seek to highlight the limits of conceptualizing fairness only in terms of group-based parity measures. Our results show that at current, making a system ``more fair'' as defined by popular metrics can harm the vulnerable social populations that were ostensibly meant to be served by the imposition of such constraints in the first place. Though the Pareto Principle is not without faults, the frequency with which ``more fair'' classification outcomes are welfare-wise dominated by ``less fair'' ones occurs is troublesome and should lead scholars to reevaluate the methodologies by which we understand the impact of machine learning on different social populations. 

\subsection{Related Work}
Research in fair machine learning has largely centered on first computationally defining ``fairness'' as a property of a classifier and then showing that techniques can be invented to satisfy such a notion \cite{kamishima2011fairness,dwork2012fairness,zemel2013learning,feldman2015certifying,zafar2015fairness,joseph2016fairness,hardt2016equality,pleiss2017fairness,calmon2017optimized,kusner2017counterfactual,kilbertus2017avoiding,zafar2017fairness,bechavod2017learning,kearns2017preventing,donini2018empirical,agarwal2018reductions}. Since most methods are meant to apply to learning problems generally, many such notions of fairness center on parity-based statistical metrics about a classifier's behavior on various protected social groups rather than on matters of utility or welfare. 

Most of the works that do look toward a social welfare-based framework for interpreting appeals to fairness sit at the intersection of computing and economics. Mullainathan \cite{mullainathan2018algorithmic} also makes a comparison between policies as set by machine learning systems and policies as set by a social planner. He argues that algorithmic systems that make explicit their description of a global welfare function are less likely to perpetrate biased outcomes and are more successful at ameliorating social inequities. Heidari et al. \cite{heidari2018fairness} propose using social welfare functions as fairness constraints on loss minimization programs. They suggest that a learner ought to optimize her classifier while in Rawls' original position. As a result, their approach to social welfare is closely tied with considerations of risk. Rather than integrate social welfare functions into the supervised learning pipeline, we keep the two perspectives distinct to emphasize the non-welfarist aspects of fair machine learning. By translating a machine learning classification into a social welfare allocation, we encourage a conception of fairness that refers to the social weights that various individuals and groups have in a society. There is a rich body of literature in social choice theory and welfare economics that investigates the normative and empirical bases of various distributions of social weights \cite{ackert2007social,fleurbaey2011theory,fleurbaey2015optimal,saez2016generalized}. This research directly bears on public policy issues ranging from tax schemes to social programs.   

The techniques that we use to perform sensitivity analysis of fairness constraints are related to a number of existing works. The proxy fairness constraint that we use in our instantiation of the $\epsilon$-fair SVM problem original appeared in Zafar et al.'s \cite{zafar2015fairness} work on restricting the disparate impact of machine classifiers. Their research introduces this particular proxy fairness constrained program and shows that it can be efficiently solved and well approximates target fairness constraints. We use the constraint to demonstrate our overall findings about the effect of fairness criteria on individual and group welfares. We share some of the preliminary formulations of our fair SVM problem with Donini et al. \cite{donini2018empirical} though they focus on the statistical and fairness guarantees of their generalized empirical risk minimization program. Lastly, though work on tuning hyperparameters of SVMs is far afield from questions of fairness and welfare, our analysis on the effect of $\Delta \epsilon$ fairness perturbations on welfare take advantage of methods in that line of research \cite{diehl2003svm,hastie2004entire,wang2007kernel}.  

\section{Problem Formalization}
\label{mapping}
%Although machine learning methods have historically been used in a variety of application areas that have little to do with problems of allocation, they are increasingly being folded into important social and economic distribution pipelines.
Our framework and results are motivated by those algorithmic uses-cases in which considerations of fairness and welfare stand alongside those of efficiency. Before we formalize the correspondence between loss minimization and welfare maximization, we first provide an overview of these two distinct perspectives on using optimization to pursue social notions of fairness.

In the empirical loss minimization task, a learner seeks a classifier $h$ that issues the most accurate predictions when trained on set of $n$ data points $\{\mathbf{x}_i, z_i, y_i\}_{i=1}^n$. Each triple gives an individual's feature vector $\bx_i \in \mathcal{X}$, protected class attribute $z_i \in \{0, 1\}$,\footnote{Though individuals in a dataset will typically be coded with many protected class attributes, in this paper we will consider only a single sensitive attribute of focus.} and true label $y_i \in \{-1, +1\}$. A model that assigns an incorrect label $h(\bx_i) \neq y_i$ incurs a penalty.

The empirical risk minimizing predictor is given by ${h^*} \defeq \argmin_{h \in \mathcal{H}} \sum_{i=1}^n \ell(h(\mathbf{x_i}), y_i)$ where hypothesis $h: \mathcal{X} \rightarrow \mathbb{R}$ gives a learner's model, the loss function $\ell: \mathbb{R} \times \{-1, +1\} \rightarrow \mathbb{R}$ gives the penalty incurred by a prediction, and $\mathcal{H}$ is the hypothesis class under the learner's consideration. In this paper, we will mainly consider $\mathcal{H}$ to be the class of separators based on hyperplane boundaries $h_{\bm{\theta}}(\mathbf{x}) = {\bm{\theta}}^\intercal \mathbf{x} + b$ with $\mathbf{x}, \bm{\theta} \in \mathbb{R}^d$ and $b \in \mathbb{R}$. For binary classification, the learner issues a prediction ${h}^{ML}(\bx) = sgn({h_{\bm{\theta}}}(\mathbf{x}))$.

Notions of fairness have been formalized in a variety of ways in the machine learning literature. Though Dwork et al.'s initial conceptualization remains prominent and influential \cite{dwork2012fairness}, much recent work has defined fairness as a parity notion applied across different protected class groups \cite{hardt2016equality, chouldechova2017fair,kleinberg2017inherent,zafar2017fairness,donini2018empirical,agarwal2018reductions}. The following definition gives the general form of these types of fairness criteria. 

\begin{definition}
\label{group-fairness}
A classifier $h$ satisfies a general group-based notion of $\epsilon$-fairness if 
\begin{align}
\abs{\mathbb{E}[g(\ell, h, \bx_i, y_i)| \mathcal{E}_{\bz_i = 1} ] - \mathbb{E}[g(\ell, h, \bx_i, y_i) |\mathcal{E}_{\bz_i = 0}]}\le \epsilon
\end{align}
where $g$ is some function of classifier $h$ performance, and $\mathcal{E}_{\bz_i = 0}$ and $\mathcal{E}_{\bz_i = 1}$ are events that occur with respect to groups $z = 0$ and $z=1$ respectively.
\end{definition}

Further specifications of the function $g$ and the events $\mathcal{E}$ instantiate particular group-based fairness notions. For example, when $g(\ell, h, \bx_i, y_i) = h(\bx_i)$ and $\mathcal{E}$ refers to the events in which $y_i = 1$ for each group, Definition 1 gives an $\epsilon$-approximation of \emph{equality of opportunity} \cite{hardt2016equality}. When $g(\ell, h, \bx_i, y_i) = \ell(h(\bx_i), y_i)$ and $\mathcal{E}$ refers to all events for each group, Definition 1 gives the notion of $\epsilon$-approximation of \emph{overall error rate balance} \cite{chouldechova2017fair}. Notice that as $\epsilon$ increases, the constraint loosens, and as $\epsilon$ decreases, the fairness constraint becomes more strict.
\\
%For our considered case of linear SVMs, we will relax $0-1$ loss and replace it with hinge loss $\ell_h = \max (0, 1-y_ih_{\bm{\theta}}(\mathbf{x}))$. 

In the Planner's Problem, a Planner maximizes a social welfare functional (SWF) given as a weighted sum of individual utilities, $W= \sum_{i=1}^n w_i u_i$. An individual $i$'s contribution to society's total welfare is a product of her utility $u_i$ and her social weight $w_i \in [0,1]$ normalized so that $\sum_i^n w_i = 1$. Utility functions $u_i: \mathcal{X} \rightarrow \mathbb{R}_+$ assign positive utilities to a set of attributes or goods $\bx_i$. We suppose a utility function is everywhere continuous and differentiable with respect to its inputs. 

Since a Planner who allocates a resource $h$ impacts her recipients' utilities, she solves ${h}^{SWF}(\mathbf{x}; \bm{w}) \defeq \argmax_{\bm{h}} \sum_{i=1}^n w_iu(\bx_i, h_i)$ under a budget constraint: $\sum_{i=1}^n h_i \le B$. Since we consider cases of social planning in which a desirable good is being allocated, it is natural to suppose that $u$ is strictly monotone with respect to $h$. As is common in welfare economics, we take $u$ to be concave in $h$, so that receiving the good exhibits diminishing marginal returns. Further, we require that the social welfare functional $W$ be symmetric: $W(\bm{h}; \bx, \bm{w}) = W(\sigma(\bm{h}); \sigma(\bx), \sigma(\bm{w}))$ for all possible permutations of $\sigma(\cdot)$. This property implies that the utility functions in the Planner's problem are not individualized. In the case of binary classification, the Planner decides whether to allocate the discrete good to individual $i$ or not ($h_i \in \{0,1\}$).
 \\
 \ignore{
\lh{Worried we still haven't addressed the complaint: But machine learning problems are explicitly not concerned with maximizing social welfare! We say, ``yes that's true, but when we're talking about fairness we \textit{are} to some degree talking about social good, so seeing how the two problems relate is a natural step.'' Maybe include this idea in intro or here somewhere.}
}
%In order to draw a correspondence between the loss minimization problem and the social welfare maximization problem, we define a matched allocation...

\section{Correspondence between Loss Minimization and Social Welfare Maximization}
To highlight the correspondence between the machine learning and welfare economic approaches to social allocation, we first show that we can understand loss minimizing solutions to also be welfare maximizing ones, albeit under a particular instantiation of the social welfare function. Since social welfare is given as the weighted sum of individuals' utilities, it is clear that manipulating weights $\bm{w}$ significantly alters the Planner's solution. Thus just as we can compute optimal allocations under a fixed set of welfare weights, we can also begin with an optimal allocation and find welfare weights that would support them. In welfare economics, the form of $\bm{w}$ corresponds to societal preferences about what constitutes a fair distribution. For example, the commonly-called ``Rawlsian'' social welfare function named after political philosopher John Rawls, can be written as $W_{Rawls} = \min_{i} u_i$ where $u_i$ gives the utility of individual $i$. This function is equivalent to the general form $\sum_{i=1}^n w_i u_i$ where the individual $i$ with the lowest utility $u_i$ has welfare weight $w_i = 1$ and all individuals $k \neq i$ have weight $w_k = 0$. On the other hand, the commonly-called ``Benthamite'' social welfare function named after the founder of utilitarianism Jeremy Bentham, aggregates social welfare such that an extra unit of utility contributes equally to the social welfare regardless of who receives it. Benthamite weights are equal across all individuals: $w_i = \frac{1}{n}$ for all $i \in [n]$.

Thus associating an optimal (possibly fairness constrained) loss minimizing allocation with a set of welfare weights that would make it socially optimal lends insight into how socially ``fair'' a classification is from a welfare economic perspective. The following Proposition formally states this correspondence between loss minimization and social welfare maximization. 

\begin{proposition}
For any vector of classifications $h^{ML}(\bx_i)$ that solves a loss minimization task, there exists a set of welfare weights $\bm{w}$ with $\sum_{i=1}^n w_i = 1$ such that the Planner who maximizes social welfare $W$ with a budget $B$ selects an optimal allocation $h^{SWF}(\bx_i) = h^{ML}(\bx_i)$ for all $i\in [n]$.
 \end{proposition}
 
\begin{proof}
First, we know that since $W(\bx, \bm{w})$ is a weighted sum of functions $u$, which are concave in $h$, the Planner can indeed find a social welfare maximizing allocation $\bm{h}^{SWF}$. Let $h^{ML}(\bx)$ be the empirical loss-minimizing classifier for $\{\bx_i, z_i, y_i\}_{i=1}^n$. With these allocations given, we can invert the social welfare maximization problem to find the weights that $\bm{w}$ support them.

For a given utility function $u$, we evaluate $\frac{\partial u(\bx, h)}{\partial h}\Bigr\rvert_{\{\bx_i, h^{ML}(\bx_i) \}}= m_i $ $\forall i \in [n]$, which gives the marginal gain in utility for individual $i$ from having received an infinitesimal additional allocation of $h$. Notice that at a welfare maximizing allocation $\bm{h}$, we must have that 
\begin{align} 
\label{optimality-condition}
w_i \frac{\partial u(\bx, h)}{\partial h}\Bigr\rvert_{\{\bx_i, h_i\}} = w_j \frac{\partial u(\bx, h)}{\partial h}\Bigr\rvert_{\{\bx_j, h_j \} } \text{ for all }i, j \in [n]
\end{align} 
When the allocation $h^{ML} (\bx)$ has been fixed, we must have that $w_i m_i = w_j m_j = k$, where the constant $k$ is set by the Planner's budget $B$, for all $i, j$ along with $\sum_{i=1}^n w_i = 1$. Since $u$ is strictly monotone with respect to $h$, $m_i > 0$ for all $i$. We thus have a non-degenerate system of $n$ equations with $n$ variables, and there exists a unique solution of welfare weights $\bm{w}$ that support the allocation. 
\end{proof}

Note that in the case of binary classification $h^{ML}(\bx) \in \{-1, +1,\}$, so allocations are not awarded at a fractional level. Thus rather than the partial $\frac{\partial u(\bx, h)}{\partial h}$, the Planner must consider the margin gain of receiving a positive classification. Nevertheless, Proposition 1 still holds, and the proof carries through with $\Delta u(\bx, h(\bx)) = u(\bx, 1) - u(\bx, 0)$ in place of partial derivatives $\frac{\partial u(\bx, h)}{\partial h}$. 

The equations given in (\ref{optimality-condition}) set an optimality condition for the Planner. Its structure, though simple, reveals that welfare weights must be inversely proportional to an individuals' marginal utility gain from receiving an allocation. This result is formalized in the Proposition below.

\begin{proposition}
\label{proposition-weights}
For any set of optimal allocations $\bm{h} = \argmax_{\bm{h}}\sum_{i=1}^n \bar{w_i} u(\bx_i, h_i)$ with strictly monotonic utility function $u$ concave in $h$, the supporting welfare weights have the form $\bar{w}_i = \frac{k}{m_i}$ where $m_i = \frac{\partial u(\bx_i)}{\partial h}\rvert_{\{\bx_i, h_i \}}$ and $k>0$ is a constant set by the Planner's budget $B = \sum_{i=1}^n h_i$. 
\end{proposition}

By associating a set of classification outcomes with a set of implied welfare weights, one can inquire about the social fairness of the allocation scheme by investigating the distribution of welfare weights across individuals or across groups. While there may not be a single distribution of welfare weights that can be said to be ``most fair,'' theoretical and empirical work in economics has been conducted on the range of fair distributions of societal weights \cite{fleurbaey2011theory,saez2016generalized}. This research has considered weights as implied by current social policies \cite{ackert2007social,zoutman2013optimal,christiansen1978implicit}, philosophical notions of justice \cite{adler2012well,fleurbaey2015optimal}, and individuals' preferences in surveys and experiments \cite{ackert2007social,kuziemko2015elastic,saez2016generalized}. They thus offer substantive notions of fairness currently uncaptured by many current algorithmic fairness approaches. 
 
\section{An Algorithm that records all Possible Labelings}
In the previous section, we showed that for any vector of classifications, one can compute the implied societal welfare weights of the generic SWF that would yield the same allocations in the Planner's Problem. In this section, we work in the converse direction: Beginning with a Planner's social welfare maximization problem, does there exist a classifier $h^{ML} \in \mathcal{H}$ that generates the same classification as the Planner's optimal allocation such that for all $i\in [n]$, $h^{ML}(\bx_i) = h^{SWF}(\bx_i)$?

We answer this question for the hypothesis class of linear decision boundary-based classifiers by providing an algorithm that accomplishes a much more general task: Given a set $\mathcal{X}$, containing $n$ $d$-dimensional nondegenerate data points $\bx \in \mathbb{R}^d $, our algorithm enumerates all linearly separable labelings and can output a hyperplane parameterized by $\bm{\theta} \in \mathbb{R}^d$ and $b\in \mathbb{R}$ that achieves that set of labels. In order to build intuition for its construction, we first consider a hyperplane separation technique that applies to a very specific case: a case in which a hyperplane separates sets $A$ and $B$, intersecting $A$ at a single point and intersecting $B$ at $d-1$ points. 	

\begin{lemma}
\label{lemma-convex}
Consider linearly separable sets $A$ and $B$ of points $\bx \in \mathbb{R}^d$. For any $d-1$-dimensional hyperplane $h_V$ with $h_V \cap A = \bv$ and $h_V \cap B = P$ where $\abs{P} = d-1$ that separates $A$ and $B$ into closed halfspaces $\bar{h}^+_V$ and $\bar{h}^-_V$, one can construct a $d-1$-dimensional hyperplane $h$ that separates $A$ and $B$ into open halfspaces ${h}^+$ and ${h}^-$.
\end{lemma}

Because its techniques are not of primary relevance for this Section, we defer the full proof of this Lemma to the Appendix but provide a brief exposition. The construction on which the Lemma relies is a ``pivot-and-translate'' maneuver. A hyperplane as described can separate points in open halfspaces by first pivoting (infinitesimally) on a $d-2$-dimensional facet $P$ of a convex hull $C(B)$ away from $\bv \in C(A)$ and then translating (infinitesimally) back toward $\bv$ and away from $C(B)$. We show that all separable convex sets can be separated by such a hyperplane and procedure.

Note that since we seek enumerations of all labelings achievable by a linear separator on a given dataset, we are not \textit{a priori} given convex hulls to separate. That is, we want to know which points \textit{can} be made into distinct convex hulls and which cannot. Thus we take the preceding procedure and invert it---the central idea is to begin with the separators and from there, search for all possible convex hulls: Beginning with an arbitrary $d-1$-dimensional hyperplane $h$ defined by $d$ data points, we construct convex hulls out of the points in each halfspace created by $h$. Then we can use the pivot-and-translate procedure to construct a separation of the two sets into two open halfspaces. We must show that such a procedure is indeed exhaustive.

\begin{algorithm}[t]
	\SetAlgoNoLine
	\KwIn{Set $\mathcal{X}$ of $n$ data points $\bx \in \mathbb{R}^d$}
	\KwOut{All possible partitions $A$, $B$ attainable via linear separators; supporting hyperplane $h$}
	\For{all $V \subset \mathcal{X}$ with $\abs{V} = d$ }{
		Construct $d-1$-dimensional hyperplane $h_V$ defined by $\bv \in V$\;
		\For{each point $\bv \in V$
		}{
			$P = V \setminus \bv$\;
			$h = pivot(h_V, P, \bv)$ \tcp*{$h_V$ pivots around the $d-2$-dimensional plane $P$ away from $\bv$}
			$h = translate(h, \bv)$ \tcp*{$h$ translates toward $\bv$}
			Record $A = \{\bx | \bx \in h^+\} , B = \{\bx | \bx \in h^- \}, h;$
		}
	}
	\caption{Record all possible labelings on a dataset $\mathcal{X}$ by linear separators}
	\label{full-alg1}
\end{algorithm}

\begin{theorem}
Given a dataset $\mathcal{X}$ consisting of $n$ nondegenerate points $\bx \in \mathbb{R}^d$, Algorithm~\ref{full-alg1} enumerates all possible labelings achievable by a $d-1$-dimensional hyperplane in $O(n^dd)$ time and outputs hyperplane parameters $(\bm{\theta}, b)$ that achieve each one. 
\end{theorem} 

\begin{proof}
We have already shown that the pivot-and-translate construction is sufficient to linearly separate two sets $A$ and $B$ in the very specific case given in Lemma 1. But we must prove that all linearly separable sets can be constructed via Algorithm 1. We prove it is exhaustive by contradiction. 

Suppose there exists a separation of $\mathcal{X}$ that is not captured by Algorithm 1. Then there exists disjoint sets $A$ and $B$ such that their convex hulls $C(A)$ and $C(B)$ do not intersect. By the hyperplane separating theorem, there exists a $d-1$-dimensional hyperplane $h_{V_1}$ that separates $A$ and $B$, defined by a set $V_1$ of $d$ vertices $\bv$, at least one of which is on the boundary of each convex hull. Without loss of generality, we assume that for all $\bx \in A$, $\bx \in h_{V_1}^+$ and for all $\bx \in B$, $\bx \in h_{V_1}^-$. Notice that this hyperplane is indeed ``checked'' by the Algorithm, and this hyperplane $h_{V_1}$ correctly separates $\bx \in \mathcal{X} \setminus V_1$ into the two sets $A$ and $B$. Thus if the separation is not disclosed via the procedure, the omission must occur due to the pivot-and-translate procedure's being incomplete. 

In Algorithm 1, the set $V_1$ is partitioned so that $V_1 = \bv_{f,1} \cup P_1$ where $\bv_{f,1}$ is the ``free vertex'' and $P_1$ is the pivot set consisting of $d-1$ vertices. This partition occurs $d$ times so that each vertex $\bv \in V_1$ has its turn as the ``free vertex.'' Thus we can view the pivot-and-translate procedure as constituting a second partition---a partition of the $d$ vertices that define the initial separating hyperplane. By contradiction, we claim that there exists a partition $D_1, E_1 \subset V_1$ such that $D_1 \coprod E_1 = V_1$ where $D_1 \subset A$ and $E_1 \subset B$ that is unaccounted for in the $d$ pivot-and-translate operations applied to $h_{V_1}$. Thus $\abs{D_1}, \abs{E_1} \ge 2$. We use a ``gift-wrapping'' argument, a technique common in algorithms that construct convex hulls, to show that the partition $A$ and $B$ is indeed covered by Algorithm 1.

Select $\bv \in D_1$ to be the free vertex $\bv_{f,1}$, and let the pivot set $P_1 = V_1 \setminus \bv_{f,1}$. We pivot around $P_1$ and away from $\bv_{f,1}$ so that $\bv_{f,1} \in h_{V_1}^+$. Rotations in $d$-dimensions are precisely defined as being around $d-2$-dimensional planes. Thus pivoting around the ridge $P_1$ away from $\bv_{f,1}$ is a well-defined rotation in $\mathbb{R}^d$. Since $h_{V_1}$ is a supporting hyperplane to $C(B)$, $E_1$ constitutes a $|E_1|-1$-dimensional facet of $C(B)$. There exists a vertex $\bv_E \in C(B)$ such that $E_1 \cup \bv_E$ gives a $\abs{E_1}$-dimensional facet of $C(B)$. Let $h_{V_2}$ be defined by the set $V_2 = P_1 \cup \bv_E$. $h_{V_2}$ continues to correctly separate all $\bx \in \mathcal{X} \setminus V_2$. 

We once again partition $V_2$ into sets $D_2$ and $E_2$ whose members must be ultimately classified in sets $A$ and $B$ respectively. Notice that $\abs{D_2} = \abs{D_1} -1$, since $h_{V_2}$ correctly classifies $\bv_{f,1}$ as belonging to set $A$. Thus with each iteration of the pivot procedure, the separating classifier unhinges from a vertex in $C(A)$ and ``wraps'' around $C(B)$ just as in the gift wrapping algorithm to attach onto another vertex in $C(B)$. At each step, the hyperplane defined by $d$ vertices continues to support and separate $C(A)$ and $C(B)$. Thus process iterates until in the $\abs{D_1}-1$-th round, the hyperplane $h_{V_{\abs{D_1}-1}}$ has partition $D_{\abs{D_1}-1}$ and $E_{\abs{D_1}-1}$ with $\abs{D_{\abs{D_1}-1}}=1$. Applying the full pivot-and-translate procedure ensures the desired separation of sets $A$ and $B$ into open halfspaces.

Thus starting from a separable hyperplane defined by $d$ vertices on the convex hulls $C(A)$ and $C(B)$, which must exist in virtue of the separability of sets $A$ and $B$, we were able to use the pivot procedure in order to ``gift-wrap'' around one convex hull until we arrived at a $d$-dimensional separating hyperplane with only one vertex $\bv_f \in C(A)$. This hyperplane is obviously checked by the first for-loop of Algorithm 1. The subsequent for-loop that performs the second partition of the $d$ vertices into the free vector $\bv_f$ and the pivot set $P$ then directly applies and performs the pivot-and-translate procedure given in Algorithm 1 to achieve the desired separation. 
\end{proof}

Degeneracies in the dataset can be handled by combining Algorithm 1 with standard solutions to degeneracy problems in geometric algorithms, which perform slight perturbations to degenerate data points to transform them into nondegenerate ones \cite{edelsbrunner1990simulation}. In concert with these solutions, Algorithm 1 automatically reveals which social welfare maximization solutions are attainable on a given dataset $\mathcal{X}$ via hyperplane-based classification and the $0-1$ accuracy loss each entails. 
%When the data are encoded with sensitive attributes $z$, each labeling can also be associated with group parity-based fairness scores. Thus the set of attainable classifications can be depicted via a number of plots that differentially emphasize vendor loss, statistical fairness quantities, and welfare weights.

\section{Sensitivity Analysis of Fairness Constraints}
In this Section, we perform welfare-minded sensitivity analyses on the standard empirical risk minimization (ERM) program with fairness constraints. Assuming, as before, that an individual benefits from receiving a positive classification, we define group utilities as 
\begin{align}
W_0= \frac{1}{n_{0}}\sum_{i | z_i = 0}\frac{h(\bx_i)+1}{2}, \qquad W_1= \frac{1}{n_{1}}\sum_{i | z_i = 1} \frac{h(\bx_i)+1}{2}
\end{align}
where $n_0$ and $n_1$ give the number of individuals in groups $z=0$ and $z=1$ respectively. 

First, we present an instantiation of the $\epsilon$-fair ERM problem with a fairness constraint proposed in prior work in algorithmic fairness. We work from a Soft-Margin SVM program and derive the various dual formulations that will be of use in the following analyses. In Section 5.2, we move on to show how $\Delta \epsilon$ perturbations to the fairness constraint yield changes in classification outcomes for individuals and by extension, how they impact a group's overall welfare. Standard sensitivity analyses show how the objective value changes as constraints are tightened and loosened, but they are unable to show how classifications themselves are affected by a changing constraint. Our approach, which draws a connection between fairness perturbations and searches for an optimal SVM regularization parameter, tracks changes in an individual's classification by taking advantage of codependence of variables in the Dual of the SVM. By perturbing the fairness constraint, we observe changes in not its own corresponding Dual variable but in the corresponding Dual of the margin constraints, which relay the classification fates of data points. Via this technique, we plot the full ``solution paths" of the Dual variable as a function of $\epsilon$ and as a result, we compute group welfares as a function of $\epsilon$. We close this Section by working from the shadow price of the fairness constraint to derive local and global sensitivities of the optimal solution to $\Delta \epsilon$ perturbations.

Our results show that tightening a fairness constraint leads to idiosyncratic changes to individuals' classification fates. We show that requiring a classifier to abide by a stricter fairness standard does not necessarily lead to improved outcomes for the disadvantaged group. Our results indicate that preferring a classifier that emits narrower parity disparities can lead to choosing outcomes that are actually Pareto dominated by seemingly ``less fair'' alternatives. In these cases, the machine learning goal of ensuring group-based fairness is incompatible with the Pareto Principle.

\begin{definition}[Pareto Principle]
\label{Pareto}
Let $x$, $y$ be two social alternatives. Let $\succeq_1, ..., \succeq_n$ be the preference ordering of individuals $i \in [n]$ and let $\succeq_P$ be the preference ordering of a Planner who maximizes social welfare. A Planner abides by the \textit{Pareto Principle} if $x \succeq_P y$ whenever $x \succeq_i y$ for all $i$.  
\end{definition}

In welfare economics, the Pareto Principle is a standard requirement of social welfare functionals---it would appear that the intentional implementation of an allocation that is Pareto dominated by an available alternative would be undesirable and even irresponsible! Nevertheless we show that in many cases, applying fairness criteria to loss minimization tasks do just that. For the sake of clarity in exposition and greater continuity with previous literature, we conduct our analysis with respect to the Soft-Margin SVM optimization problem, however the analyses and results in this Section can be applied to fairness-constrained convex loss minimization programs more generally.

\subsection{Setting up the $\epsilon$-fair ERM program}

%dual variable corresponding to the fairness constraint as giving the shadow price of fairness. We then derive local and global sensitivities to $\Delta \epsilon$ fairness perturbations. The shadow price of a fairness constraint shows the increase in a learner's loss that would result if the fairness constraint were tightened, but the quantity reveals nothing about how the individuals who receive classifications are affected by a tightened fairness constraint. Thus in Section 4.2, we move on to show how $\Delta \epsilon$ perturbations yield changes in classification outcomes for individuals and by extension, how they impact a group's overall welfare. Assuming, as before, that an individual benefits from receiving a positive classification, we define group utilities as 

The general fairness-constrained empirical loss minimization program can be written as 
\begin{mini}[s]
  {h\in \mathcal{H}}{\ell(h(\bx), y) }{}{}
\addConstraint{f_h(\bx, y)}{ \leq \epsilon}
\end{mini}
where $\ell(h(\bx), y)$ gives the empirical loss of a classifier $h\in \mathcal{H}$ on the dataset $\mathcal{X}$. To maximize accuracy, the learner ought to minimize 0-1 loss; however because the function ($\ell_{0-1}$)  is non-convex, it is difficult to minimize. A convex surrogate loss such as hinge loss ($\ell_h$) or log loss ($\ell_{\log}$) is frequently substituted in its place to ensure that globally optimal solutions may be efficiently found. $f_{h} (\bx, y) \le \epsilon$ gives a group-based fairness constraint of the type given in Definition \ref{group-fairness}. $\epsilon > 0$ is the unfairness ``tolerance parameter''---a greater $\epsilon$ permits greater group disparity on a metric of interest; a smaller $\epsilon$ more tightly restricts the level of permissible disparity.

We examine the behavior of fairness-constrained linear SVM classifiers. In particular, our learner minimizes hinge loss with $L_1$ regularization; equivalently she, seeks a Soft-Margin SVM that is ``$\epsilon$-fair.'' The fair empirical risk minimization program that will be of central interest is given as
\begin{mini*}
  {\bm{\theta}, b}{ \frac{1}{2}\norm{\bm{\theta}}^2 + C\sum_{i=1}^n \xi_i }{}{} 
  \label{original-fair-ERM}
  \addConstraint{y_i(\bm{\theta}^\intercal \bx_i + b) - 1 +\xi_i}{\geq 0} \tag{$\epsilon$-fair Soft-SVM}
\addConstraint{\xi_i}{ \geq 0}
\addConstraint{f_{\bm{\theta}, b}(\bx, y)}{ \leq \epsilon}
\end{mini*}
where the learner seeks linear hyperplane parameters $\bm{\theta}, b$; $\xi_i$ are non-negative slack variables that violate the margin constraint in the Hard-Margin SVM problem $y_i (\bm{\theta}^\intercal \bx_i + b) - 1 \ge 0$, and $C > 0$ is a hyperparameter tunable by the learner to optimize the trade-off between preferring a larger margin and penalizing violations of the margin. 

The abundant literature on algorithmic fairness presents a long menu of options for the various forms that $f_{\bm{\theta}}$ could take, but generally speaking, the constraints are non-convex and and thus require other methods of training classifiers that deviate from directly pursuing efficient fairness constraint-based convex programming methods \cite{kamishima2011fairness,bechavod2017learning,kearns2017preventing,agarwal2018reductions,zafar2017fairness}. In response, researchers have devised convex proxy alternatives, which have been shown to approximate the results of the original fairness constraints well \cite{zafar2015fairness,donini2018empirical,woodworth2017learning}. Since we will use the well-tread machinery of convex optimization, we primarily work with these convex fairness constraints. In particular, we will work with the proxy constraint proposed by Zafar et al. \cite{zafar2015fairness} which disallows disparities in covariance between group membership and the (signed) distance between individuals' feature vectors and the hyperplane decision boundary that exceed $\epsilon$. The fairness constraint is written as
\begin{align}
\label{covariance}
 f_{\bm{\theta}, b} =\abs{\frac{1}{n}\sum_{i=1}^n (z_i -\bar{z})(\bm{\theta}^\intercal\bx_i+ b)} \leq \epsilon 
 \end{align}

 %covariance constraint as a metric that approximates unfairness  We will perform sensitivity analysis apply their constraint to the standard empirical risk minimization program given in (\ref{original-fair-ERM}) thus gives
\ignore{
\begin{mini*}
  {\bm{\theta}, b}{ \frac{1}{2}\norm{\bm{\theta}}^2 + C\sum_{i=1}^n \xi_i }{}{}
  \label{fair-SVM}
  \addConstraint{y_i(\bm{\theta}^\intercal \bx_i + b) - 1 +\xi}{\geq 0} \tag{$\epsilon$-fair-SVM1}
\addConstraint{\abs{\frac{1}{n}\sum_{i=1}^n (z_i -\bar{z})(\bm{\theta}^\intercal\bx_i+ b)}}{ \leq \epsilon}
\end{mini*}
}
\noindent where $\bar{z}$ reflects the bias in the demographic makeup of $\mathcal{X}$: $\bar{z} = \frac{1}{n} \sum_{i=1}^n z_i$. Let ($\epsilon$-fair-SVM1-P) be the Soft-Margin SVM program with this covariance constraint. The corresponding Lagrangian is
\begin{align}
\label{lagrangian-primal}
\mathcal{L}_P(\bm{\theta},b, \bm{\xi},\bm{\lambda}, \bm{\mu},\gamma_1, \gamma_2)  &= \frac{1}{2} \norm{\bm{\theta}}^2 + C\sum_{i=1}^n \xi_i- \sum_{i=1}^n \lambda_i - \sum_{i=1}^n \mu_i (y_i (\bm{\theta}^\intercal \bx_i + b)- 1 +{\xi_i}) \tag{$\epsilon$-fair-SVM1-L}
\\
&- \gamma_1\big(\epsilon - \frac{1}{n}\sum_{i=1}^n(z_i - \bar{z})(\bm{\theta}^\intercal \bx_i+b)\big) - \gamma_2\big(\epsilon - \frac{1}{n}\sum_{i=1}^n(\bar{z} - z_i)(\bm{\theta}^\intercal \bx_i + b)\big) \nonumber
\end{align}
where $\bm{\theta} \in \mathbb{R}^d, b\in \mathbb{R}, \bm{\xi} \in \mathbb{R}^n$ are Primal variables. The (non-negative) Lagrange multipliers $\bm{\lambda},\bm{\mu} \in \mathbb{R}^n$ correspond to the $n$ non-negativity constraints $\xi_i \ge 0$ and the margin-slack constraints $y_i(\bm{\theta}^\intercal \bx_i + b) - 1 +\xi_i \ge 0$ respectively. The multipliers $\gamma_1, \gamma_2 \in \mathbb{R}$ correspond to the two linearized forms of the absolute value fairness constraint. By complementary slackness, dual variables reveal information about the satisfaction or violation of their corresponding constraints. The sensitivity analyses in the subsequent two subsections will focus on these interpretations.

By the Karush-Kuhn-Tucker conditions, at the solution of the convex program, the gradients of $\mathcal{L}$ with respect to $\bm{\theta}$, $b$, and ${\xi_i}$ are zero. Plugging in these conditions, the Dual Lagrangian is
\ignore{
\begin{align*} \frac{\partial \mathcal{L}}{\partial \bm{\theta}} &\coloneqq 0 \Rightarrow \bm{\theta} = \sum_{i=1}^n \mu_i y_i \bx_i - \frac{\gamma}{n}(\sum_{i=1}^n (z_i - \bar{z}) \bx_i)
\\
\frac{\partial \mathcal{L}}{\partial b} &\coloneqq 0 \Rightarrow \sum_{i=1}^n \mu_i y_i =  \frac{\gamma}{n} \sum_{i=1}^n (z_i - \bar{z}) = 0
\\
\frac{\partial \mathcal{L}}{\partial \xi_i} &\coloneqq 0 \Rightarrow \lambda_i  + \mu_i = C, \qquad  i=1, \ldots ,n \end{align*}
}
\begin{align}
\label{lagrangian-dual}
\mathcal{L}_D(\bm{\theta},\bm{\xi},\bm{\lambda}, \bm{\mu},\gamma_1, \gamma_2)  &= -\frac{1}{2} \norm{\sum_{i=1}^n \mu_i y_i \bx_i - \frac{\gamma}{n}\sum_{i=1}^n (z_i - \bar{z}) \bx_i}^2 +\sum_{i=1}^n \mu_i - \abs{\gamma}\epsilon
\end{align}
where $\gamma = \gamma_1 - \gamma_2$. Thus the Dual maximizes this objective subject to the constraints $\mu_i \in [0,C]$ for all $i$ and $\sum_{i=1}\mu_i y_i = 0$. We thus derive the full Dual problem 
\begin{maxi*}
  {\bm{\mu}, \gamma}{-\frac{1}{2}\norm{\sum_{i=1}^n \mu_i y_i \bx_i - \frac{\gamma}{n}\sum_{i=1}^n (z_i - \bar{z}) \bx_i}^2 +\sum_{i=1}^n \mu_i - V\epsilon}{}{}
 {\label{dual-SVM-fair}}
\addConstraint{\mu_i}{\in [0,C],}{\qquad i=1, \ldots ,n} \tag{$\epsilon$-fair-SVM1-D}
\addConstraint{\sum_{i=1}^n \mu_i y_i}{= 0}
\addConstraint{\gamma}{\in [-V,V]}
\end{maxi*}
where we have introduced the variable $V$ to eliminate the absolute value function $\abs{\gamma}$ in the objective. 
Notice that when $\gamma = 0$ and neither of the fairness constraints bind, we recover the standard dual SVM program. Since we are concerned with fairness constraints that alter an optimal solution, we are interested in cases in which $V$ is strictly positive. As such, we can rewrite the preceding as
%\begin{align*}
%\mathcal{L}_{DD}(\bm{\mu},\gamma, &V, \beta_-, \beta_+, \alpha_-, \alpha_+) = \\ &-\frac{1}{2}\norm{\sum_{i=1}^n \mu_i y_i (I - P_\bu)\bx_i}^2 +\sum_{i=1}^n \mu_i +\frac{N\sum_{i}\mu_i y_i \langle \bx_i, \bu \rangle}{\norm{\bu}^2}(\beta_- - \beta_+) + \epsilon(2\beta_-)
%\end{align*}
\ignore{
\begin{maxi*}
 {\substack{\bm{\mu}, \sigma, \beta_-, \beta_+,\\ \bm{\alpha_-}, \bm{\alpha_+}}}{ -\frac{1}{2}\norm{\sum_{i=1}^n \mu_i y_i (I - P_\bu)\bx_i}^2+\sum_{i=1}^n (1+\alpha_{-,i} - \alpha_{+,i} + \sigma y_i) \mu_i }{}{}
\breakObjective{\qquad +\frac{N\sum_{i}\mu_i y_i \langle \bx_i, \bu \rangle}{\norm{\bu}^2}(\beta_- - \beta_+) + \epsilon(2\beta_-)- C\sum_{i=1}^n \alpha_{-,i}} 
\addConstraint{\alpha_i, \alpha_+, \beta_-, \beta_+, \sigma}{\geq 0}{\qquad i=1, \ldots ,n} \tag{$\epsilon$-fair SVM2-D}
\label{SVM2-D}
\addConstraint{\beta_- + \beta_+}{= \epsilon}
\addConstraint{\mu_i y_i (I - P_\bu)\bx_i}{\in [-V,V]}
\end{maxi*}}

\begin{maxi*}
 {\substack{\bm{\mu}, \beta_-, \beta_+}}{ -\frac{1}{2}\norm{\sum_{i=1}^n \mu_i y_i (I - P_\bu)\bx_i}^2+\sum_{i=1}^n \mu_i +\frac{2n\sum_{i}\mu_i y_i \langle \bx_i, \bu \rangle + n^2(\beta_- - \beta_+)}{2\norm{\bu}^2}(\beta_- -\beta_+)}{}{}
\addConstraint{\mu_i}{\in [0,C],}{\qquad i=1, \ldots ,n} \tag{$\epsilon$-fair SVM2-D}
\label{SVM2-D}
\addConstraint{\sum_{i=1}^n \mu_i y_i}{= 0}
\addConstraint{\beta_-, \beta_+}{ \geq 0}
\addConstraint{\beta_- + \beta_+}{ = \epsilon}
\end{maxi*}
where $I, P_\bu \in \mathbb{R}^{d\times d}$. The former is the identity matrix, and the latter is the projection matrix onto the vector defined by $\bu = \sum_{i=1}^n (z_i - \bar{z})\bx_i$. As was also observed by Donini et al., the $\epsilon = 0$ version of (\ref{SVM2-D}) is thus equivalent to the standard formulation of the dual SVM program with Kernel $K(\bx_i, \bx_j) = \langle (I-P_\bu)\bx_i, (I-P_\bu)\bx_j \rangle$ \cite{donini2018empirical}.

\subsection{Sensitivity on Candidates}
In this Section, we investigate the effects of perturbing a fixed $\epsilon$-fair SVM by some $\Delta \epsilon$ on the classification outcomes that are issued. We ask, \textit{``How are groups' classifications, and thus their utilities, impacted when a learner tightens or loosens her fairness constraint?''} The insight is that rather than perform sensitivity analysis directly on the Dual variable corresponding to the fairness constraint---which, as we will see in Section 4.3, only gives information about the change in the learner's objective value---we track changes in the classifier's behavior by analyzing the effect of $\Delta \epsilon$ on another set of Dual variables: $\mu_i$ that correspond to the Primal margin constraints. We harness techniques that have been used in finding SVM regularization solution paths to demarcate the range of perturbations that yield outcomes that either improve or worsen a group's utility \cite{hastie2004entire,wang2007kernel,diehl2003svm}.
 
We show that perturbations to $\epsilon$ do not necessarily correspond to meaningful changes in group utilities. We find that decreasing $\epsilon$, which corresponds to making the Soft-Margin SVM ``more fair,'' need not translate into improved utilities. In fact, in terms of welfare, we show that policies that na{\"i}vely prefer ``more fair'' classifier solutions do not abide by the Pareto Principle defined in (\ref{Pareto}). Perturbations of $\epsilon$ to $\epsilon + \Delta \epsilon < \epsilon$, which tighten the fairness constraint, do not generally translate into improved outcomes for either of the two groups. And since a learner's loss never decreases when a fairness condition is made more strict and classifier outcomes can make both groups worse-off, then optimal SVM classifiers that are subject to more ``unfair'' constraint can yield classifications that Pareto dominate those that arise under more ``fair'' conditions. That is, every stakeholder group prefers the outcomes issued by the ``unfair" classifier. 
\\ 

%By writing $\epsilon$ as a perturbation parameter, we can analyze the change in the value of $p(\epsilon)$ as we make perturbations from a given $\epsilon$ to $\epsilon + \Delta \epsilon$. 

%Because the SVM is a quadratic program, we are guaranteed that strong duality holds, and we can thus directly work with the formulation (\ref{SVM2-D}).

%This formulation it is clear that $p(\epsilon)$ decreases as a function of $\epsilon$, with the rate of change dictated by $\gamma$, the effect of $\epsilon$ on the utilities of each of the two groups cannot be read off of the SVM dual problem alone. 

Define a function $p(\epsilon): \mathbb{R} \rightarrow \mathbb{R}$ that assigns the optimal value of the $\epsilon$-fair loss minimizing program ($\epsilon$-fair-SVM1-P). We begin at a solution $p(\epsilon)$ and consider classifications at the solution $p(\epsilon + \Delta \epsilon)$, where $\Delta \epsilon$ can be either positive or negative. For clarity of exposition, we assume that the positive covariance fairness constraint binds, and thus that $\gamma = V > 0$. This is without loss of generalization---the same analyses apply when $\gamma = V < 0$. The Dual $\epsilon$-fair SVM program is thus

\begin{mini*}
 {\bm{\mu}}{\frac{1}{2}\norm{\sum_{i=1}^n \mu_i y_i (I - P_\bu)\bx_i}^2 -\sum_{i=1}^n \mu_i +\frac{n \epsilon(2\sum_{i}\mu_i y_i \langle \bx_i, \bu \rangle - n \epsilon)}{2\norm{\bu}^2}}{}{}
\addConstraint{\mu_i}{\in [0,C],}{\qquad i=1, \ldots ,n} \tag{$\epsilon$-fair SVM-D}
\label{SVM-D}
\addConstraint{\sum_{i=1}^n \mu_i y_i}{= 0}
\end{mini*}
At the optimal solution, the classification fate of each data point $\bx_i$ is encoded in the dual variable $\mu_i^*$. Let $D$ be the value of the objective function in (\ref{SVM-D}), then we have that
\begin{align}
\label{free}
    \frac{\partial D}{\partial \mu_j^*} > 0 &\longrightarrow \mu_j^* = 0, \text{ and } j \in \mathcal{F}\\
\label{support}
    \frac{\partial D}{\partial \mu_j^*} = 0 &\longrightarrow \mu_j^* \in [0,C], \text{ and } j \in \mathcal{S}\\
\label{error}
     \frac{\partial D}{\partial \mu_j^*} < 0 &\longrightarrow \mu_j^* = C, \text{ and } j \in \mathcal{E}
\end{align}
Partitioning the dataset $\mathcal{X}$ based on $\frac{\partial D}{\partial \mu_j^*}$ at any optimal solution, $\bx_j$ are either free vectors (\ref{free}), support vectors in the margin (\ref{support}), or error vectors (\ref{error}). To analyze the impact that applying a fairness constraint has on a group's welfare, we can track the behavior of $\frac{\partial D}{\partial \mu_i}$ and observe how vectors' membership in sets $\mathcal{F}$, $\mathcal{S}$, and $\mathcal{E}$ change under a perturbation to $\epsilon$. This information will in turn reveal how classifications change or are stable upon tightening or loosening a fairness constraint. 

Fairness perturbations are not guaranteed to shuffle data points across the different membership sets $ \mathcal{F}, \mathcal{S}$, and $\mathcal{E}$. It is clear that for $j\in \{\mathcal{F}, \mathcal{E}\}$, so long as a perturbation does not cause $\frac{\partial D}{\partial \mu_j^{\epsilon}}$ to flip signs or to vanish, then $j$ will belong to the same set and $h^\epsilon(\bx_j) = h^{\epsilon+\Delta \epsilon}(\bx_j)$ where $h^\epsilon(\bx_j)$ gives the $\epsilon$-fair classification outcome for $\bx_j$. In these cases, a candidate's welfare is unaffected by the change in the fairness tolerance level. In contrast, support vectors $\bx_j$ with $j \in \mathcal{S}$ are subject to a different condition to ensure that they stay in the margin: $\frac{\partial D}{\partial \mu_i^\epsilon} = \frac{\partial D}{\partial \mu_i^{\epsilon + \Delta \epsilon}} = 0$. So we have that 
 %$\bx_j$ in the margin where $j \in \mathcal{S}$ are most vulnerable to experiencing a change in classification such that $h^\epsilon(\bx_j) \neq h^{\epsilon +\Delta \epsilon}(\bx_j)$. Before we analyze the effects of $\Delta \epsilon$ on vectors that undergo classification changes, we first 
\begin{align}
\label{dD-dm}
    \frac{\partial D}{\partial \mu_j^\epsilon} = \sum_{i=1}^n \mu_i y_i (I -P_\bu) \bx_i y_j (I-P_\bu)\bx_j + \frac{n \epsilon y_j \langle \bx_j, \bu \rangle}{\norm{\bu}^2} + by_j - 1 = 0
\end{align}
Let $r_j \Delta \epsilon$ be the change in $\mu_j$ upon perturbing $\epsilon$ by $\Delta \epsilon$, then we have 
\begin{align}
\label{update-r}
\mu_j^{\epsilon + \Delta \epsilon} = \mu_j^{\epsilon} + r_j\Delta \epsilon
\end{align}
where $\mu_j^\epsilon$ is the optimal $\mu_j$ value at the optimal solution $p(\epsilon)$. Let $r_0$ be the change in the offset $b$; then we can solve for $\bm{r_j} \in \mathbb{R}^{n+1}$ for all unchanging $\bx_j \in \mathcal{S}$ by taking the finite difference of (\ref{dD-dm}) with respect to a $\Delta \epsilon$ perturbation,
\begin{align*}
    \sum_{i=1}^n r_i \Delta \epsilon y_i y_j \langle (I-P_\bu) \bx_i, (I-P_\bu) \bx_j\rangle + r_0y_j = \frac{-n y_j \Delta \epsilon}{\norm{\bu}^2}\langle \bu, \bx_j \rangle
\end{align*}
It is clear that for all $i \in \{\mathcal{F}, \mathcal{S}\}$, the corresponding $\mu_i^\epsilon$ sensitivity to perturbations must have $r_i \Delta \epsilon = 0$, so $r_i = 0$ for all $i$. We can then simplify the previous expression by summing only over those $r_i$ where $i\in \mathcal{S}$.
\begin{align*}
    \sum_{i \in \mathcal{S}} r_i \Delta \epsilon y_i y_j \langle (I-P_\bu) \bx_i, (I-P_\bu) \bx_j\rangle + r_0y_j = \frac{-n y_j \Delta \epsilon}{\norm{\bu}^2}\langle \bu, \bx_j \rangle
\end{align*}
Thus $\bm{r_j}$ can be found by inverting the matrix 
\begin{align}
K=
\left(
\renewcommand{\arraystretch}{2}
\begin{array}{c|cccc}
0 & y_1  & y_2  & \dots & y_{|\mathcal{S}|}\\
\cline{1-5}
y_1 & & & & \\
\vdots & \multicolumn{4}{|c} {y_i y_j \langle (I-P_\bu) \bx_i, (I-P_\bu) \bx_j\rangle} \\
y_2 & & & & \\
y_{|\mathcal{S}|} & & & &
\end{array}
\right) \in \mathbb{R}^{(|\mathcal{S}|+1) \times (|\mathcal{S}|+1)}
\end{align}
where indices are renumbered to reflect only those $i, j \in \mathcal{S}$. This matrix is invertible so long as the Kernel $K(\bx_i, \bx_j) = \langle(I-P_\bu) \bx_i, (I-P_\bu)\bx_j \rangle$ forms a positive definite matrix. Since the objective function in (\ref{SVM-D}) is quadratic, then a sufficient condition for the Kernel matrix to be invertible is that it is strictly convex---we assume this as a technical condition. Then the sensitivities of $\mu_j$ for $j \in \mathcal{S}$ to $\Delta \epsilon$ perturbations are given by 
\begin{align}
\label{r}
    \bm{r}= K^{-1}\Big(\frac{-n}{\norm{\bu}^2}\bv\Big), \hspace{3 pt} \text{ where } \bv = \begin{bmatrix}
0 \\
\vdots \\
 y_j \langle \bu, \bx_j \rangle \\
\vdots\\
\end{bmatrix} \in \mathbb{R}^{|\mathcal{S}|+1}
\end{align}
The sensitivities $r_j \neq 0$ for $j\in \mathcal{S}$ do affect the quantities $\frac{\partial D}{\partial \mu_j}$ for all $j \in [n]$, and thus we need additional conditions to hold to ensure that the vectors not on the margin are also unshuffled by the fairness perturbation. Define 
\begin{align}
\label{d}
d_j = \frac{\partial D}{\partial \mu_j \partial \epsilon} = \sum_{i \in \mathcal{S}} r_i y_i y_j \langle(I-P_\bu) \bx_i, (I-P_\bu) \bx_j \rangle +r_0 y_j 
\end{align}
and the quantity of interest for stability of vectors $\bx_j$ for $j \notin \mathcal{S}$ is then given by
\begin{align}
\dfrac{ \frac{\partial D}{\partial \mu_j^\epsilon} }{d_j} \gtrless 0 
\end{align}
where $>$  entails that $j \in \mathcal{F}$ and $<$ entails that $j \in \mathcal{E}$. Now we bound $\Delta \epsilon$ such that no vectors are shuffled across different sets. It follows that perturbations in this range do not alter classifications. 

\begin{proposition}
\label{stable-perturbations}
Let $p(\epsilon)$ be the optimal $\epsilon$-fair SVM loss and denote the optimal $\bm{\mu}^*$ at $p(\epsilon)$ as $\bm{\mu^\epsilon}$. Let $d_j = \frac{\partial D}{\partial \mu_j \partial \epsilon}$ and $g_j = 1 - \Big(\sum_{i=1}^n \mu^\epsilon_i y_i (I-P_\bu)\bx_i y_j (I-P_\bu)\bx_j + \frac{n \epsilon y_j \langle \bx_j, \bu \rangle}{\norm{\bu}^2} + b y_j\Big)$.  All perturbations of $\epsilon$ in the range $\Delta \epsilon \in \big( \max_j m_j, \min_j M_j\big)$ where
\begin{align}
\label{range-stable}
m_j = \begin{cases} \begin{cases} {\frac{g_j}{d_j}}, & j \in \mathcal{F},  d_j > 0 \\ -\infty, &  j \in \mathcal{F},  d_j < 0 \end{cases} \\ \min \{\frac{C-\mu_j^\epsilon}{r_j}, \frac{-\mu^\epsilon_j}{r_j}\}, & j \in \mathcal{S} \\ \begin{cases} -\infty, & j\in \mathcal{E}, d_j >0 \\ {\frac{g_j}{d_j}}, & j \in \mathcal{E},  d_j  <0 \end{cases} \end{cases}, \qquad M_j = \begin{cases} \begin{cases} \infty, & j \in \mathcal{F},  d_j > 0 \\ {\frac{g_j}{d_j}}, &  j \in \mathcal{F},  d_j < 0 \end{cases} \\ \min \{\frac{C-\mu_j^\epsilon}{r_j}, \frac{-\mu^\epsilon_j}{r_j}\}, & j \in \mathcal{S} \\ \begin{cases} {\frac{g_j}{d_j}}, & j\in \mathcal{E}, d_j >0 \\ \infty, & j \in \mathcal{E},  d_j  <0 \end{cases}\end{cases} 
\end{align}
yield no changes to memberships in the partition $\{\mathcal{F}, \mathcal{S}, \mathcal{E}\}$.
\end{proposition}
%classifications $h(\bx_i)$ for all $i\in [n]$ issued by the optimal SVM solution to $p(\epsilon + \Delta \epsilon)$. 

We defer the interested reader to the Appendix for the full proof of this Proposition. The result follows from observing that for $i \in \mathcal{F}$, any perturbations $\Delta \epsilon$ that increase $\frac{\partial D}{\partial \mu^\epsilon_i}$ do not threaten $i$'s exiting $\mathcal{F}$; if $\Delta \epsilon$ leads to a decrease in $\frac{\partial D}{\partial \mu^\epsilon_i}$, then $i$ can enter $\mathcal{S}$. Inversely, perturbations $\Delta \epsilon$ that decrease $\frac{\partial D}{\partial \mu^\epsilon_i}$ ensure that $i\in \mathcal{E}$ stay in the same partition, but perturbations that increase $\frac{\partial D}{\partial \mu^\epsilon_i}$ can cause $i$ to shuffle into $\mathcal{S}$. Support vectors in the margin must maintain $\mu_i^{\epsilon +\Delta \epsilon} \in [0,C]$. Once $\mu_i^\epsilon$ hits either endpoint of the interval, the vector $\bx_i$ risks shuffling across to $\mathcal{F}$ or $\mathcal{E}$. Computing these transition inequalities results in a set of conditions that ensure that a partition is stable. Since $\Delta \epsilon$ can be either positive or negative, we take the maximum of the lower bounds and the minimum of the upper bounds to arrive at the range of stable perturbations given in (\ref{range-stable}).\\

This Proposition reveals a surprising ineffectiveness of fairness constraints. So long as the fairness constraint is binding and its associated dual variable $\gamma > 0$, then tightening or loosening a fairness constraint \textit{does} alter the loss of the optimal learner classifier---the actual SVM solution changes---yet analyzed from the perspective of the individual agents $\bx_i$, so long as the $\Delta \epsilon$ perturbation occurs within the range given by (\ref{range-stable}), classifications issued under this $\epsilon+\Delta \epsilon$-fair SVM solution are identical to those under the $\epsilon$-fair solution. Thus despite the apparent more `fair'' signal that a classifier abiding by $\epsilon+\Delta \epsilon$ sends, agents are no better off in terms of welfare. This result is summarized in the following Corollary. 

\begin{corollary} 
\label{corollary-stable}
Let $\{p(\epsilon), W_0(\epsilon), W_1(\epsilon)\}$ be a triple expressing the utilities of the learner, group $z=0$, and group $z=1$ under the $\epsilon$-fair SVM solution. Then for any $\Delta \epsilon \in(\max_j m_j, 0)$ where $m_j$ is defined in (\ref{range-stable}), $\{p(\epsilon), W_0(\epsilon), W_1(\epsilon)\} \succsim \{p(\epsilon+\Delta \epsilon), W_0(\epsilon+\Delta \epsilon), W_1(\epsilon+\Delta \epsilon)\}$.
\end{corollary} 

By demarcating the limits of $\Delta \epsilon$ perturbations that yield no changes to the sets $\mathcal{F}, \mathcal{S}, \mathcal{E}$, we can move on to consider the effects of perturbations $\Delta \epsilon$ that exceed the stable region given by (\ref{range-stable}). There are four ways that vectors can be shuffled across the partition:
\begin{enumerate}
\item $j \in \mathcal{E}^\epsilon$ moves into $\mathcal{S}^{\epsilon + \Delta \epsilon}$ 
\\
\item $j \in \mathcal{F^\epsilon}$ moves into $\mathcal{S}^{\epsilon + \Delta \epsilon}$
\\
\item $j \in \mathcal{S}^\epsilon$ moves into $\mathcal{F}^{\epsilon + \Delta \epsilon}$
\\
\item $j \in \mathcal{S}^\epsilon$ moves  into $\mathcal{E}^{\epsilon + \Delta \epsilon}$ 
\end{enumerate}

At each ``breakpoint'' event when $\Delta \epsilon$ reaches $\max_j m_j$ or $\min_j M_j$, the set $\mathcal{S}$ changes, and $r_j$ for all $j \in \mathcal{S}$ must be recomputed via (\ref{r}). The new $r_j$ sensitivities hold until the next breakpoint. 

\begin{lemma}
\label{lemma-piecewise}
$\mu_i^\epsilon$ for all $i \in [n]$ are piecewise linear in $\epsilon$.
\end{lemma}
We defer the full proof to the Appendix but provide a brief exposition of the result. If perturbations $\Delta \epsilon$ are in the stable region given in (\ref{range-stable}), then for all $\mu_j(\epsilon)$ with $j \in \{\mathcal{F}, \mathcal{E}\}$, $\mu_j(\epsilon) = \mu_j(\epsilon +\Delta \epsilon)$. For points $\bx_j$ that are in the margin and thus $j \in \mathcal{S}$, $\mu_j(\epsilon+\Delta \epsilon) = \mu_j(\epsilon) +  r_j\Delta \epsilon$. Since index transitions at the breakpoint only occur by way of the margin, showing that the $\mu_i(\epsilon)$ paths are continuous is sufficient in order to conclude that they are piecewise linear.  \\

By parameterizing dual variables $\mu_j(\epsilon)$, we can associate a group utility with the optimal classification scheme of each of the $\Delta \epsilon$ perturbation breakpoints. As already illustrated, partitions are static in the stable regions around each breakpoint, so group utilities will also be unchanged in these regions. As such, we can directly compare group utilities at neighboring breakpoints. Of the four possible events that occur a breakpoint, index transitions between the partitions $\mathcal{S}$ and $\mathcal{E}$ correspond to changed classifications that thus affect group utilities. The following Proposition characterizes those breakpoint transitions that effect utility triples for group $A$, group $B$, and the learner $\{p(\epsilon),W_0(\epsilon), W_1(\epsilon)\}$ that are strictly Pareto dominated by the utility triple supported at a neighboring $\epsilon$ breakpoint. The full proof is left to the Appendix.

\begin{proposition}
\label{compare-utility}
Consider the utility triple at the optimal $\epsilon$-fair SVM solution given by $\{p(\epsilon),W_0(\epsilon), W_1(\epsilon)\}$. Let $b_L = \max_j m_j < 0$ be the neighboring lower breakpoint, and let $b_U = \min_j M_j > 0$ be the neighboring upper breakpoint. If $b_L =\frac{g_j}{d_j}$ where $j \in \mathcal{E}^\epsilon$ and $y_j = -1$ or if $b_L = \frac{C - \mu_j^\epsilon}{r_j}$ where $j\in \mathcal{S}^\epsilon$ and $y_j = +1$, then \[\{p(\epsilon + b_L),W_0(\epsilon + b_L), W_1(\epsilon+b_L)\}\} \prec \{p(\epsilon),W_0(\epsilon), W_1(\epsilon)\}\]
Let $b_U = \frac{g_j}{d_j}$ where  $j\in \mathcal{E}^\epsilon$ and $y_j = +1$ or if $b_U = \frac{C - \mu_j^\epsilon}{r_j}$ where $j\in \mathcal{S}^\epsilon$ and $y_j = -1$, then \[\{p(\epsilon + b_U),W_0(\epsilon+b_U), W_1(\epsilon+b_U)\}\succ \{p(\epsilon),W_0(\epsilon), W_1(\epsilon)\}\]
 \end{proposition}

Since $\epsilon$ breakpoints allow the comparison of group utilities that result from various $\epsilon$-fair SVM solutions, we can track the solution paths of the $\mu_i(\epsilon)$ for all individuals $i$ in a group $z$ in order to construct a single curve of group $z$'s welfare that is also parameterized by the fairness tolerance level $\epsilon$. Algorithm ~\ref{full-alg2} and Algorithm ~\ref{full-alg3} (in the Appendix) gives an implementation that constructs solution paths $\mu_i(\epsilon)$ and outputs the curves tracking group welfare.

\begin{algorithm}[t]
	\SetAlgoNoLine
	\KwIn{set $\mathcal{X}$ of $n$ data points $\{\bx_i, z_i, y_i\}$}
	\KwOut{solutions paths $\bm{\mu}(\epsilon)$ and group welfare curves $\{W_0(\epsilon), W_1(\epsilon)\}$}
	$\bm{\mu}^0$ =  $\argmin_{\bm{\mu}} D(\bm{\mu})$ of (0-fair SVM-D)\;
	$\mathcal{F} = \emptyset$, $\mathcal{S} = \emptyset$, $\mathcal{E} = \emptyset$\;
	$\epsilon = 0$, $\Delta \epsilon = 0$\;
	\While{$\epsilon < 1$}{
		\For{each $\mu^\epsilon_i$}{
			update $\mathcal{F}, \mathcal{S}, \mathcal{E}$ according to (\ref{free}), (\ref{support}), (\ref{error})
			}
		compute $\bm{r}$, $\bm{d}$ according to (\ref{r}), (\ref{d})\;
		$\Delta \epsilon = \min_i M_i$ as given in (\ref{range-stable})\;
		$\mu_i^{\epsilon+\Delta \epsilon} = \mu_i^\epsilon +r_i \Delta \epsilon $ for $i\in \mathcal{S}$, $\mu_i^\epsilon = \mu_i^{\epsilon+\Delta \epsilon}$ for $i\in \mathcal{F}, \mathcal{E}$\;
		$\epsilon = \epsilon + \Delta \epsilon$\;
		$\{W_0(\epsilon), W_1(\epsilon)\} = gp\_welfare(\mu, \by,\bz)$ \tcp*{calls Algorithm 3 in Appendix to compute group welfare}
		return $(\epsilon, \bm{\mu})$
	}
	\caption{Sensitivity Analysis of Group Welfares to Changing $\epsilon$}
	\label{full-alg2}
\end{algorithm}

We thus find that minimizing loss in the presence of stricter fairness constraints does not correspond to monotonic gains or losses in the welfare level of candidate groups. As a result, fairness perturbations do not have a straightforwardly predictable effect on classification decisions at all! Further, these results do not only arise as an unfortunate outcome of using the particular proxy fairness constraint suggested by Zafar et al \cite{zafar2015fairness}. In fact, so long as the $\epsilon$ parameter appears in the linear part of the dual soft-margin SVM objective function, the $\bm{\mu}$ paths will exhibit a piecewise linear form that can be similarly characterized by stable regions and breakpoints. Thus these results apply to many of the risk minimization programs subject to proxy fairness criteria that have been proposed in the literature \cite{donini2018empirical,woodworth2017learning,zafar2015fairness}. Further, even when the dual variable paths $\mu_i(\epsilon)$ are not piecewise linear, so long as they are non-monotonic, fairer classification outcomes do not necessarily confer welfare benefits to the disadvantaged group.\\

The preceding analyses show that although fairness constraints are often intended to improve classification outcomes for some disadvantaged group, they in general do not abide by the Pareto Principle, a common welfare economic axiom for deciding among social alternatives. That is, asking that an algorithmic procedure abide by a more stringent fairness criteria can lead to enacting classification schemes that actually make every stakeholder group worse-off, both social groups as well as the learner. Here, the supposed ``improved fairness'' achieved by decreasing the unfairness tolerance parameter $\epsilon$ fails to translate into any meaningful improvements in the number of desirable outcomes issued to members of either group. 

\begin{theorem}
\label{pareto-thm}
Consider two fairness-constrained ERM programs parameterized by $\epsilon_1$ and $\epsilon_2$ where $\epsilon_1 < \epsilon_2$. Then a decision-maker who always prefers the classification outcomes issued under the ``more fair'' $\epsilon_1$-fair solution to those under the ``less fair'' $\epsilon_2$-fair solution does not abide by the Pareto Principle. 
\end{theorem}

\begin{figure}[ht]
\centering
\includegraphics[scale=0.75]{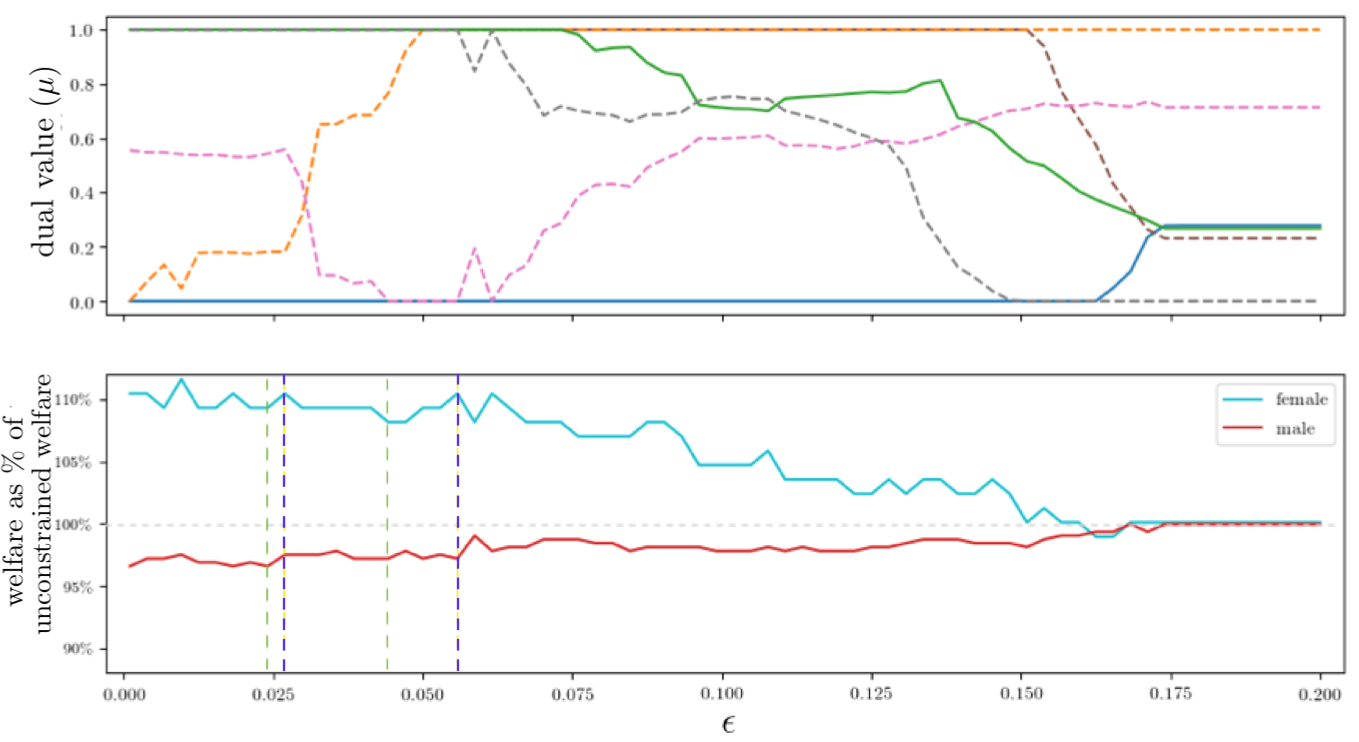}
\caption{Sensitivity analysis of $\epsilon$-fair SVM-solution on Adult dataset. Increasing $\epsilon$ from left to right loosens fairness constraint, and classification outcomes become ``less fair.'' Paths level off at $\epsilon \approx 0.175$ when constraint ceases to bind at the optimal solution. \textbf{Top Panel}: Example ``paths'' of dual variables $\mu_i$ as a function of $\epsilon$. For $\mu_i = 0$, $i\in \mathcal{F}$ (correctly labeled); for $\mu_i \in (0,1)$, $i \in \mathcal{S}$ (correctly labeled, in margin); for $\mu_i = 1$, $i\in \mathcal{E}$ (incorrectly labeled). Paths are piecewise linear, show changes in classification of $\bx_i$ as $\epsilon$ changes. Solid paths are coded female; dotted paths are coded male. \textbf{Bottom Panel}: Relative group-specific welfare change at $\epsilon$-fair SVM solution given as percentage of group welfare at the unconstrained SVM solution. Corresponding plot of absolute welfare changes is given in the Appendix. Non-monotonicity of welfare curves shows that always preferring more fair solutions does not abide by the Pareto Principle. Green dashed vertical lines to the left of purple ones give examples of classifications that are Pareto-dominated but ``more fair.''}
\end{figure}

\subsection{Sensitivity Analysis on Learner Optimal Value}
Having proven the main welfare-relevant sensitivity result for groups, we return to conduct more standard analysis of the effect of $\Delta \epsilon$ perturbations on the learner's loss. Recall that in this case, we directly inquire the Dual variable of the fairness constraint. Solving for $\gamma^*$ in (\ref{SVM2-D}) yields
\begin{align}
\label{shadow-price-fairness}
    \gamma^* = \frac{n(n(\beta_--\beta_+)+ \sum_{i=1}^n \mu_i y_i \langle \bx_i, \bu \rangle)}{\norm{\bu}^2}
\end{align}
By complementary slackness, one of $\beta_-$ and $\beta_+$ is zero, while the other is equal to $\epsilon$. In particular, if $\beta_- = 0$, then $\beta_+ = \epsilon$, and we know that $\gamma = V > 0$. Thus the original fairness constraint that binds is the upper bound on covariance, suggesting that the optimal classifier must be constrained to limit its positive covariance with group $z = 1$. Similarly, if $\beta_+ = 0$, then $\gamma = -V < 0 $, and the optimal classifier must be constrained to limit its positive covariance with group $z=0$. 
\ignore{
\begin{align*}
\gamma > 0 &\longrightarrow  \sum_{i=1}^n \mu_i y_i \langle \bx_i, \bu\rangle > n \epsilon \\
\gamma < 0 &\longrightarrow  \sum_{i=1}^n \mu_i y_i \langle \bx_i, \bu\rangle > -n \epsilon
\end{align*}
}

We can interpret the value of the Dual variable Lagrange multiplier $\gamma^*$ given in (\ref{shadow-price-fairness}) as the shadow price of the fairness constraint. It gives the additional loss in accuracy that the learner would achieve if the fairness constraint were infinitesimally loosened. Whenever a fairness constraint binds, its shadow price is readily computable and is given by 
\begin{align}
\label{shadow-price}
    \abs{\gamma} = \frac{n\abs{n\epsilon + \sum_{i=1}^n \mu_i y_i \langle \bx_i, \bu \rangle}}{\norm{\bu}^2}
\end{align}

It bears noting that because (\ref{original-fair-ERM}) is not a linear program, (\ref{shadow-price}) can onl be interpreted as a measure of \textit{local} sensitivity, valid only in a small neighborhood around an optimal solution. But through an alternative lens of sensitivity analysis, we can derive a lower bound on global sensitivity due to changes in the fairness tolerance parameter $\epsilon$. By writing $\epsilon$ as a perturbation variable, we can perform sensitivity analysis on the same $\epsilon$-constrained problem. Returning to the perturbation function $p(\epsilon)$, we have
\begin{align}
\label{perturb}
p(\epsilon) \ge \sup_{\bm{\mu}, \gamma}\{\mathcal{L}(\bm{\mu}^*, \gamma^*) - \epsilon \abs{\gamma^*}\}
\end{align}
where $\mathcal{L}(\bm{\mu}^*, \gamma^*)$ gives the optimal solution to the SVM problem with $\epsilon = 0$:
\begin{align}
\mathcal{L}(\bm{\mu}^*, \gamma^*) = \max_{\bm{\mu} \in [0,C]^n, \gamma} -\frac{1}{2} \norm{\sum_{i=1}^n \mu_i y_i (I - P_u)\bx_i}^2+ \sum_{i=1} \mu_i
\end{align}
The perturbation formulation given in (\ref{perturb}) is identical in form to the original program ($\epsilon$-fair-SVM1-P) but gives a global bound on $p(\epsilon)$ for all $\epsilon \in [0,1]$. Since (\ref{perturb}) gives a lower bound, the global sensitivity bound yields an asymmetric interpretation.   

\begin{proposition}
\label{global-prop}
If $\Delta \epsilon < 0$ and $\abs{\gamma^*} \gg 0$, then $p(\epsilon + \Delta \epsilon) - p(\epsilon) \gg 0$. If $\Delta \epsilon > 0$ and $\abs{\gamma^*}$ is small, then $p(\epsilon + \Delta \epsilon) - p(\epsilon) < 0$ but small in magnitude.
\end{proposition}

Proposition \ref{global-prop} reveals that tightening the fairness constraint when the shadow price of the fairness constraint is high leads to a great increase in vendor loss, but loosening the fairness constraint when the shadow price is small leads only to a small decrease in loss. 

\section{Discussion}
As algorithmic systems increasingly make life-shaping social and economic decisions, researchers in machine learning must reevaluate both their lodestars of optimality and efficiency as well as their latest metrics of fairness. Since notions of fairness are invariably context-dependent and always informed by background normative views, it is unsurprising that there has been such wide disagreement within the community about which of the many fairness definitions is the ``right'' one. In reality, the search space is much larger, and there is no objective winner. 

This paper does not look to offer another fairness definition. Instead, by viewing classification outcomes as allocations of a good, we incorporate considerations of individual and group utility in our analysis of classification regimes. The role of the concept of ``utility'' in evaluations of social policy has been controversial since Bentham popularized the notion over 200 years ago. But in many cases of social distribution, utility considerations provide a partial but still important perspective on what is at stake within a task of distribution. An individual needs various social and economic resources throughout the course of her life; utility-based notions of welfare can capture the relative benefit that a particular good can have on a particular individual. If machine learning systems are in effect serving as resource distribution mechanisms, then questions about fairness should align with questions of ``Who benefits?'' Our results show that many parity-based formulations of fairness in machine learning do not ensure that disadvantaged groups benefit. Working to ensure that a classifier better accords with a fairness measure can lead to selecting allocations that lower the welfare for every group (as well as the learner). There are several reasons that favor limiting levels of inequality that are not reflected in utilitarian calculus, but without acknowledging and accounting for these reasons, well-intentioned optimization tasks that seek to be ``fairer'' can further disadvantage social groups for no reason but to satisfy a given fairness metric. 

We propose that the field of algorithmic fairness look to work in welfare economics for both specific insights into formalizing substantive notions of fairness in distribution and also general insights into how to build a ``technical'' field and methodology that more effectively grapples with normative questions. Welfare economics exists as a branch of economics that is explicitly concerned with what public policies \textit{ought} to be, how to maximize individuals' well-beings, and what types of distributive outcomes are preferable. Answers to these questions appeal to values and judgments that do not refer only to descriptive or predictive facts about a state of affairs. It would appear that the success of fair machine will largely hang on how well it can adapt to a similar ambitious task. 

\bibliographystyle{unsrt}
\bibliography{welfare-pareto}
\pagebreak
\section{Appendix}

\subsection{Dual derivations of the $\epsilon$-fair SVM program}
In this Appendix section, we walk through the preliminary setup of the $\epsilon$-fair SVM program given in Section 5.1 and present intermediate derivations omitted from the main text. 

Recall that the fair empirical risk minimization program of central focus is
\begin{mini*}
  {\bm{\theta}, b}{ \frac{1}{2}\norm{\bm{\theta}}^2 + C\sum_{i=1}^n \xi_i }{}{} 
  \label{original-fair-ERM}
  \addConstraint{y_i(\bm{\theta}^\intercal \bx_i + b) - 1 +\xi_i}{\geq 0} \tag{$\epsilon$-fair Soft-SVM}
\addConstraint{\xi_i}{ \geq 0}
\addConstraint{f_{\bm{\theta}, b}(\bx, y)}{ \leq \epsilon}
\end{mini*}
The linear hyperplane parameters are $\bm{\theta} \in \mathbb{R}^d$ and $b\in \mathbb{R}$. The non-negative $\xi_i$ allow the margin constraints to have some slack---this is why these variables are commonly called ``slack variables.'' In the Soft-Margin (as opposed to the Hard-Margin) SVM, the margin is permitted to be less than 1. A slack variable $\xi_i > 0$ corresponds to a point $\bx_i$ having a functional margin of less than 1. There is a cost associated with this margin violation, even though it need not correspond to a classification error. $C > 0$ is a hyperparameter tunable by the learner to optimize this trade-off between preferring a larger margin and penalizing violations of the margin. 

When we combine the general Soft-Margin SVM with the the covariance constraint in (\ref{covariance}) proposed by Zafar et al. \cite{zafar2015fairness}, we have the program
\begin{mini*}
  {\bm{\theta}, b}{ \frac{1}{2}\norm{\bm{\theta}}^2 + C\sum_{i=1}^n \xi_i }{}{}
  \label{fair-SVM}
  \addConstraint{y_i(\bm{\theta}^\intercal \bx_i + b) - 1 +\xi}{\geq 0} \tag{$\epsilon$-fair-SVM1-P}
\addConstraint{\abs{\frac{1}{n}\sum_{i=1}^n (z_i -\bar{z})(\bm{\theta}^\intercal\bx_i+ b)}}{ \leq \epsilon}
\end{mini*}
\noindent where $\bar{z}$ reflects the bias in the demographic makeup of $\mathcal{X}$: $\bar{z} = \frac{1}{n} \sum_{i=1}^n z_i$. The corresponding Lagrangian is
\begin{align}
\label{lagrangian-primal}
\mathcal{L}_P(\bm{\theta},b, \bm{\xi},\bm{\lambda}, \bm{\mu},\gamma_1, \gamma_2)  &= \frac{1}{2} \norm{\bm{\theta}}^2 + C\sum_{i=1}^n \xi_i- \sum_{i=1}^n \lambda_i - \sum_{i=1}^n \mu_i (y_i (\bm{\theta}^\intercal \bx_i + b)- 1 +{\xi_i}) \tag{$\epsilon$-fair-SVM1-L}
\\
&- \gamma_1\big(\epsilon - \frac{1}{n}\sum_{i=1}^n(z_i - \bar{z})(\bm{\theta}^\intercal \bx_i+b)\big) - \gamma_2\big(\epsilon - \frac{1}{n}\sum_{i=1}^n(\bar{z} - z_i)(\bm{\theta}^\intercal \bx_i + b)\big) \nonumber
\end{align}
where $\bm{\theta} \in \mathbb{R}^d, b\in \mathbb{R}, \bm{\xi} \in \mathbb{R}^n$ are Primal variables. The (non-negative) Lagrange multipliers $\bm{\lambda},\bm{\mu} \in \mathbb{R}^n$ correspond to the $n$ non-negativity constraints $\xi_i \ge 0$ and the margin-slack constraints $y_i(\bm{\theta}^\intercal \bx_i + b) - 1 +\xi_i \ge 0$ respectively. The multiplier $\mu_i$ relays information about the functional margin of its corresponding point $\bx_i$. If the margin is greater than 1 in the Primal, \textit{i.e.}, there is slack in the constraint), then by complementary slackness, $\mu_i = 0$. Otherwise, if the constraint holds with equality, $\mu_i \in (0, C]$. When the classifier commits an error on $\bx_i$, $y_i(\bm{\theta}^\intercal \bx_i + b) \le$, and then by the KKT conditions, $\mu_i = C$. 

The multipliers $\gamma_1, \gamma_2 \in \mathbb{R}$ correspond to the two linearized forms of the absolute value fairness constraint. Notice that these two constraints cannot simultaneously hold with equality for $\epsilon >0$. Thus, by complementary slackness again, we know that at least one of $\gamma_1, \gamma_2$ is zero, and the other is strictly positive. 

By the Karush-Kuhn-Tucker conditions, at the solution of the convex program, the gradients of $\mathcal{L}$ with respect to $\bm{\theta}$, $b$, and ${\xi_i}$ are zero:
\begin{align*} \frac{\partial \mathcal{L}}{\partial \bm{\theta}} &\coloneqq 0 \Rightarrow \bm{\theta} = \sum_{i=1}^n \mu_i y_i \bx_i - \frac{\gamma}{n}(\sum_{i=1}^n (z_i - \bar{z}) \bx_i)
\\
\frac{\partial \mathcal{L}}{\partial b} &\coloneqq 0 \Rightarrow \sum_{i=1}^n \mu_i y_i =  \frac{\gamma}{n} \sum_{i=1}^n (z_i - \bar{z}) = 0
\\
\frac{\partial \mathcal{L}}{\partial \xi_i} &\coloneqq 0 \Rightarrow \lambda_i  + \mu_i = C, \qquad  i=1, \ldots ,n \end{align*}
Plugging in these conditions, the Dual Lagrangian is
\begin{align}
\label{lagrangian-dual}
\mathcal{L}_D(\bm{\theta},\bm{\xi},\bm{\lambda}, \bm{\mu},\gamma_1, \gamma_2)  &= -\frac{1}{2} \norm{\sum_{i=1}^n \mu_i y_i \bx_i - \frac{\gamma}{n}\sum_{i=1}^n (z_i - \bar{z}) \bx_i}^2 +\sum_{i=1}^n \mu_i - \abs{\gamma}\epsilon
\end{align}
where $\gamma = \gamma_1 - \gamma_2$. Thus the Dual maximizes this objective subject to the constraints $\mu_i \in [0,C]$ for all $i$ and $\sum_{i=1}\mu_i y_i = 0$. We thus derive the full Dual problem 
\begin{maxi*}
  {\bm{\mu}, \gamma}{-\frac{1}{2}\norm{\sum_{i=1}^n \mu_i y_i \bx_i - \frac{\gamma}{n}\sum_{i=1}^n (z_i - \bar{z}) \bx_i}^2 +\sum_{i=1}^n \mu_i - V\epsilon}{}{}
 {\label{dual-SVM-fair}}
\addConstraint{\mu_i}{\in [0,C],}{\qquad i=1, \ldots ,n} \tag{$\epsilon$-fair-SVM1-D}
\addConstraint{\sum_{i=1}^n \mu_i y_i}{= 0}
\addConstraint{\gamma}{\in [-V,V]}
\end{maxi*}
where we have introduced the variable $V$ to eliminate the absolute value function $\abs{\gamma}$ in the objective. 
Notice that when $\gamma = 0$ and neither of the fairness constraints bind, we recover the standard dual SVM program. Since we are concerned with fairness constraints that alter an optimal solution, we are interested in cases in which $V$ is strictly positive. As such, we can rewrite the preceding by plugging in the optimal $\gamma^*$ as given in (\ref{shadow-price-fairness}):
\begin{align*}
    \gamma^* = \frac{n(n(\beta_--\beta_+)+ \sum_{i=1}^n \mu_i y_i \langle \bx_i, \bu \rangle)}{\norm{\bu}^2}
\end{align*}
%\begin{align*}
%\mathcal{L}_{DD}(\bm{\mu},\gamma, &V, \beta_-, \beta_+, \alpha_-, \alpha_+) = \\ &-\frac{1}{2}\norm{\sum_{i=1}^n \mu_i y_i (I - P_\bu)\bx_i}^2 +\sum_{i=1}^n \mu_i +\frac{N\sum_{i}\mu_i y_i \langle \bx_i, \bu \rangle}{\norm{\bu}^2}(\beta_- - \beta_+) + \epsilon(2\beta_-)
%\end{align*}
\ignore{
\begin{maxi*}
 {\substack{\bm{\mu}, \sigma, \beta_-, \beta_+,\\ \bm{\alpha_-}, \bm{\alpha_+}}}{ -\frac{1}{2}\norm{\sum_{i=1}^n \mu_i y_i (I - P_\bu)\bx_i}^2+\sum_{i=1}^n (1+\alpha_{-,i} - \alpha_{+,i} + \sigma y_i) \mu_i }{}{}
 \breakObjective{\qquad +\frac{N\sum_{i}\mu_i y_i \langle \bx_i, \bu \rangle}{\norm{\bu}^2}(\beta_- - \beta_+) + \epsilon(2\beta_-)- C\sum_{i=1}^n \alpha_{-,i}} 
\addConstraint{\alpha_i, \alpha_+, \beta_-, \beta_+, \sigma}{\geq 0}{\qquad i=1, \ldots ,n} \tag{$\epsilon$-fair SVM2-D}
\label{SVM2-D}
\addConstraint{\beta_- + \beta_+}{= \epsilon}
\end{maxi*}
}
Thus we can write (\ref{dual-SVM-fair}) as
\begin{maxi*}
 {\substack{\bm{\mu}, \beta_-, \beta_+}}{ -\frac{1}{2}\norm{\sum_{i=1}^n \mu_i y_i (I - P_\bu)\bx_i}^2+\sum_{i=1}^n \mu_i +\frac{2n\sum_{i}\mu_i y_i \langle \bx_i, \bu \rangle + n^2(\beta_- - \beta_+)}{2\norm{\bu}^2}(\beta_- -\beta_+)}{}{}
\addConstraint{\mu_i}{\in [0,C],}{\qquad i=1, \ldots ,n} \tag{$\epsilon$-fair SVM2-D}
\label{SVM2-D}
\addConstraint{\sum_{i=1}^n \mu_i y_i}{= 0}
\addConstraint{\beta_-, \beta_+}{ \geq 0}
\addConstraint{\beta_- + \beta_+}{ = \epsilon}
\end{maxi*}
where $I, P_\bu \in \mathbb{R}^{d\times d}$. The former is the identity matrix, and the latter is the projection matrix onto the vector defined by $\bu = \sum_{i=1}^n (z_i - \bar{z})\bx_i$. As was also observed by Donini et al., the $\epsilon = 0$ version of (\ref{SVM2-D}) is thus equivalent to the standard formulation of the dual SVM program with Kernel $K(\bx_i, \bx_j) = \langle (I-P_\bu)\bx_i, (I-P_\bu)\bx_j \rangle$ \cite{donini2018empirical}.

\subsection{Algorithms}

\begin{algorithm}[h]
	\SetAlgoNoLine
	\KwIn{dual variables $\bm{\mu}$, true labels $\by$, group memberships $\bz$}
	\KwOut{group welfares $\{W_0, W_1\}$}
		$W_{0} = 0$\;
		$W_{1} = 0$\;
		$\abs{n_0} = \sum_{i=1}^n \mathds{1}[z_i = 0]$\;
		$\abs{n_1} = \sum_{i=1}^n \mathds{1}[z_i = 1]$\;
		\For{each $\mu_i$}{	
			\If{($\mu_i < C$ \& $y_i = 1)$ $||$ $(\mu_i = C$ \& $y_i = 0)$}{
				$W_{z_i} = W_{z_i} + 1$\;
			}
		}
		return $(\frac{W_0}{n_0}, \frac{W_1}{n_1})$
	\caption{Compute Group Welfares}
	\label{full-alg3}
\end{algorithm}

\subsection{Additional Figures}
\begin{figure}[h]
\centering
\includegraphics[scale=0.8]{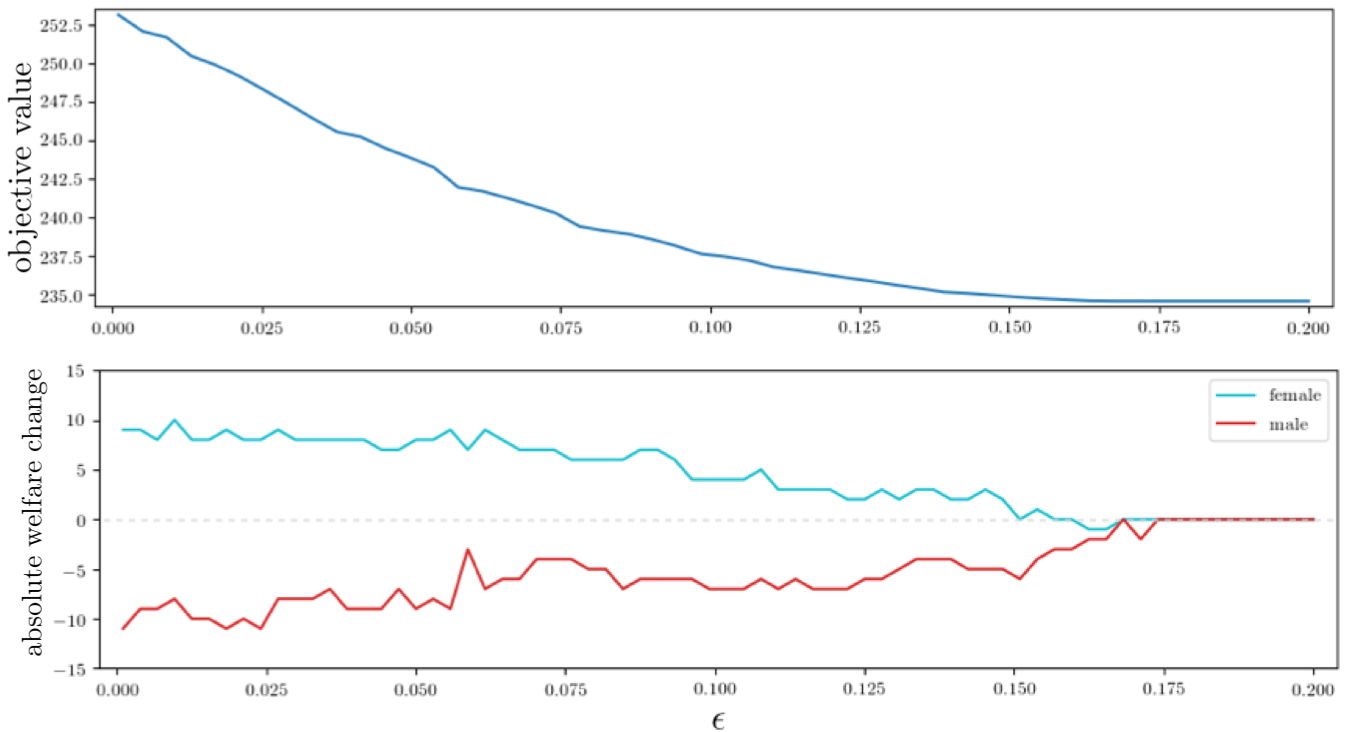}
\caption{Sensitivity analysis of $\epsilon$-fair SVM-solution on Adult dataset. Increasing $\epsilon$ from left to right loosens fairness constraint, and classification outcomes become ``less fair.'' Paths level off at $\epsilon \approx 0.175$ when constraint ceases to bind at the optimal solution.\textbf{Top Panel}: Learner objective value monotonically decreases as fairness constraint loosens. \textbf{Bottom Panel}: Absolute group-specific welfare change at $\epsilon$-fair SVM solution given as absolute change in the number of positively labeled examples compared to the group unconstrained baseline.}
\end{figure}
\subsection{Proofs}
\subsubsection{Proof of Lemma \ref{lemma-convex}}
\begin{proof}
Let $A$ and $B$ be a pair of disjoint non-empty convex sets that partition $\mathcal{X} \subset \mathbb{R}^d$: $A \coprod B = \mathcal{X}$. Then by the hyperplane separation theorem, there exists a pair ($\bm{\theta},b$) such that for all $\bx \in A$, $\bm{\theta}^\intercal \bx \ge b$---call this closed halfspace $\bar{h}^+$---and for all $\bx \in B$, $\bm{\theta}^\intercal \bx \le b$---call this closed halfspace $\bar{h}^-$. One such hyperplane can be constructed to separate the convex hulls of $A$ and $B$
\[ C(A) = \big\{ \sum_{i=1}^{|A|} \alpha_i \bx_i | \bx_i \in A, \alpha_i \ge 0, \sum_{i=1}^{|A|} \alpha_i = 1 \big\} \qquad C(B) =  \big\{ \sum_{i=1}^{|B|} \alpha_i \bx_i | \bx_i \in B, \alpha_i \ge 0, \sum_{i=1}^{|B|} \alpha_i = 1 \big\} \]

Let $h_{V}$ be the $d-1$-dimensional hyperplane defined by the set $V$ with $|V| = d$ such that $V \cap C(A) \neq \emptyset$ and $V \cap C(B) \neq \emptyset$. In order for the hyperplane to separate $C(A)$ and $C(B)$, $h_V$ must also support each hull---we know that such a hyperplane always exists. In order to separate $C(A)$ and $C(B)$ so they are contained within open halfspaces $h_{V}^+$ and $h_{V}^-$, we wiggle the hyperplane so that it no longer passes through vertices $\bv \in V$ but still maintains convex hull separation. This ``wiggle'' step is the final step of separating $A$ and $B$. 

Suppose $V$ can be partitioned into a single vertex $\bv_A$ in $C(A)$ and a set $P = \{ \bv | \bv \in C(B)\}$ with $\abs{P} = d-1$. The set $P$ defines a ridge on $C(B)$, since it is a $d-2$-dimensional facet of $C(B)$. Rotations in $d$-dimensions are precisely defined as being around $d-2$-dimensional planes. Thus pivoting $h_V$ around the ridge $P$ away from $\bv_A$ is a well-defined rotation in $\mathbb{R}^d$. Selecting any infinitesimally small rotation angle $\rho$ will be enough to have $C(A)\in h_V^+$. After the pivot, we translate $h_V$ away from the ridge $P$ back toward $\bv_A$. An infinitesimal translation is sufficient, since we simply wish to dislodge $h_V$ from the ridge $P$, so that $C(B) \in h_V^-$.
\end{proof}
\subsubsection{Proof of Proposition \ref{global-prop}} 
\begin{proof}
Following much of the exposition in the main text, recall we have that the perturbation function in (\ref{perturb}) is given as 
\[p(\epsilon) \ge \sup_{\bm{\mu}, \gamma}\{\mathcal{L}(\bm{\mu}^*, \gamma^*) - \epsilon \abs{\gamma^*}\}\]
which gives a global lower bound. Thus when a perturbation $\Delta \epsilon < 0$ causes $\mathcal{L}(\bm{\mu}^*, \gamma^*) - \epsilon \abs{\gamma^*}$ to increase, then $p(\epsilon + \Delta \epsilon)$ is guaranteed to increase by at least $\Delta \epsilon \abs{\gamma^*}$. Thus when $\abs{\gamma^*} \gg 0$, $p(\epsilon + \Delta \epsilon) -p(\epsilon) \gg 0$. The learner experience a significant increase in her optimal value $p(\epsilon)$ (which she wishes to minimize). 

On the other hand, when $\Delta \epsilon > 0$, then $\mathcal{L}(\bm{\mu}^*, \gamma^*) - \epsilon \abs{\gamma^*}$ decreases. But the decrease gives only the lower bound, and thus when $\abs{\gamma^*}$ is small, her optimal value $p(\epsilon)$ decreases but it is guaranteed not to decrease by much. 
 \end{proof}
 
\subsubsection{Proof of Proposition \ref{stable-perturbations}}
\begin{proof}
For all $j \in \mathcal{F}^\epsilon$, remaining in $\mathcal{F}^{\epsilon + \Delta \epsilon}$ after the perturbation requires that $\frac{\partial D}{\partial \mu_j} >  0$ after the perturbation. Let $\mu_i^\epsilon$ be the optimal $\mu_i$ solution at $p(\epsilon)$. Then following (\ref{dD-dm}), we rewrite the quantity $\frac{\partial D}{\partial \mu_j}$ as
\[g_j = 1 - \Big(\sum_{i=1}^n \mu^\epsilon_i y_i (I-P_\bu)\bx_i y_j (I-P_\bu)\bx_j + \frac{n \epsilon y_j \langle \bx_j, \bu \rangle}{\norm{\bu}^2} + b y_j\Big) < 0\]
If $d_j \Delta \epsilon > 0$, then $j \in \mathcal{F}^{\epsilon + \Delta \epsilon}$. Otherwise, for $d_j \Delta \epsilon < 0$, if ${\Delta \epsilon} < {\frac{g_j}{d_j}}$, then $\frac{\partial D}{\partial \mu_j^{\epsilon +\Delta \epsilon}} > 0$, and $j \in \mathcal{F}^{\epsilon + \Delta \epsilon}$ after the perturbation. \checkmark

\noindent The same reasoning follows for $j \in \mathcal{E}^\epsilon$, except we have that $g_j > 0$. Thus if $d_j \Delta \epsilon < 0$, then $j \in \mathcal{E}^{\epsilon + \Delta \epsilon}$. Otherwise, for $d_j \Delta \epsilon > 0$, if ${\Delta \epsilon} < \frac{g_j}{d_j}$, then $\frac{\partial D}{\partial \mu_j^{\epsilon +\Delta \epsilon}} > 0$, and $j \in \mathcal{E}^{\epsilon + \Delta \epsilon}$ after the perturbation. \checkmark

To ensure that support vectors do not escape the margin, we can directly look to $r_j = \frac{\partial \mu_j}{\partial \epsilon}$. Since for all $j \in \mathcal{S}^\epsilon$, $\mu_j^\epsilon \in [0,C]$, then staying in the margin and set $\mathcal{S}^{\epsilon + \Delta \epsilon}$ depends on the sign of $r_j$ and requires that
\begin{align} 
r_j < 0 &\longrightarrow \frac{C-\mu_j^\epsilon}{r_j}< \Delta \epsilon < \frac{-\mu_j^\epsilon}{r_j}\\
r_j > 0 &\longrightarrow \frac{-\mu^\epsilon_j}{r_j} < \Delta \epsilon < \frac{C-\mu^\epsilon_j}{r_j}
\end{align} 
Thus taking the minimum of the positive quantities gives an upper bound, while taking the maximum of the negative quantities gives a lower bound on $\Delta \epsilon$ perturbations, such that \{$\mathcal{F}, \mathcal{S}, \mathcal{E}\}^\epsilon = \{\mathcal{F}, \mathcal{S}, \mathcal{E}\}^{\epsilon + \Delta \epsilon}$. Let 
\[m_j = \begin{cases} \begin{cases} {\frac{g_j}{d_j}}, & j \in \mathcal{F},  d_j > 0 \\ -\infty, &  j \in \mathcal{F},  d_j < 0 \end{cases} \\ \min \{\frac{C-\mu_j^\epsilon}{r_j}, \frac{-\mu^\epsilon_j}{r_j}\}, & j \in \mathcal{S} \\ \begin{cases} -\infty, & j\in \mathcal{E}, d_j >0 \\ {\frac{g_j}{d_j}}, & j \in \mathcal{E},  d_j  <0 \end{cases} \end{cases}, \qquad M_j = \begin{cases} \begin{cases} \infty, & j \in \mathcal{F},  d_j > 0 \\ {\frac{g_j}{d_j}}, &  j \in \mathcal{F},  d_j < 0 \end{cases} \\ \min \{\frac{C-\mu_j^\epsilon}{r_j}, \frac{-\mu^\epsilon_j}{r_j}\}, & j \in \mathcal{S} \\ \begin{cases} {\frac{g_j}{d_j}}, & j\in \mathcal{E}, d_j >0 \\ \infty, & j \in \mathcal{E},  d_j  <0 \end{cases}\end{cases}  \]
Thus all perturbations of $\epsilon$ within the range\[\Delta \epsilon \in \big(\max_j{m_j}, \min_j{M_j}\big)\] satisfy the necessary conditions to ensure stable sets $\{\mathcal{F},\mathcal{S},\mathcal{E}\}$. Stable classifications $\hat{y}_i$ follow. 
\end{proof}

\subsubsection{Proof of Corollary \ref{corollary-stable}}
\begin{proof}
For all $\Delta \epsilon$ in the stable region given in (\ref{range-stable}), $W_i(\epsilon) = W_i(\epsilon +\Delta \epsilon)$ where $i$ gives group membership $z = i$. Thus the groups are welfare-wise indifferent between classifications at $\epsilon$ and $\Delta \epsilon$. For all $\Delta \epsilon < 0$, where the fairness constraint is tightened,$p(\epsilon) \le p(\epsilon +\Delta \epsilon)$. Since the learner prefers lower loss, we have that $p(\epsilon) \succeq p(\epsilon +\Delta \epsilon)$. Comparing the triples at each $\epsilon$ value, we thus have
\[\{p(\epsilon), W_0(\epsilon), W_1(\epsilon) \} \succeq \{p(\epsilon+\Delta \epsilon), W_0(\epsilon+ \Delta\epsilon), W_1(\epsilon + \Delta \epsilon) \}\]
as desired.
\end{proof}
\subsubsection{Proof of Lemma \ref{lemma-piecewise}}

\begin{proof}
Consider the dual variables $\bm{\mu^{\epsilon_0}}$ at the optimal SVM solution $p(\epsilon_0)$. By Proposition (\ref{stable-perturbations}), for all perturbations $\Delta \epsilon_0 \in (\max_i m_i, \min_i M_i)$, $\mu_i$ for $i\in \mathcal{S}$ change according to (\ref{update-r}), which is clearly linear in $\Delta \epsilon$; for all $i \notin \mathcal{S}$, $\frac{\partial \mu_i}{\partial \epsilon} = 0$, and $\mu_i^{\epsilon_0} = \mu_i^{\epsilon_0 +\Delta \epsilon_0}$. \\\
For $\Delta \epsilon_0$ perturbations beyond this range, at least one vector $\bx_i$ leaves its original set at $\epsilon_0$ and enters another at $\epsilon_0 + \Delta \epsilon_0$. Without loss of generality consider the upper bound to the $\Delta \epsilon_0$ stability region, denoted $b_0 = \min_i M_i$. We want to show that all possible $\mu_i$ paths at this breakpoint are continuous and piecewise linear. Consider $j \in \mathcal{E}^{\epsilon_0}$ moving into $j \in \mathcal{S}^{\epsilon_0 + b_0}$: $\mu_j^{\epsilon_0} = \mu_j^{\epsilon_0 + \Delta \epsilon_0} = C$ for all $\Delta \epsilon_0 \in [\epsilon_0, \epsilon_0+b_0)$. At $\Delta \epsilon_0 =b_0$, we update the partition such that $\mathcal{S}^{\epsilon_0 +b_0} =  S^{\epsilon_0} \cup \{j\}$ and recompute  $\bm{r_j}$ for all $j \in \mathcal{S}$. Following (\ref{update-r}), we consider new perturbations $\Delta \epsilon_1$ from $\epsilon_1 = \epsilon_0 + b_0$.  thus have that $\mu_j^{\epsilon_1+ \Delta \epsilon_1} = C + r_j \Delta \epsilon_1$, which is linear in $\Delta \epsilon$. The same argument follows for $j\in \mathcal{F}^\epsilon_0$ moving into $j\in \mathcal{S}^{\epsilon_0 + b_0}$ where $\mu_j^{\epsilon_0} = \mu_j^{\epsilon_0 + \Delta \epsilon_0} = 0$ for all $\Delta \epsilon_0 \in (\epsilon_0, \epsilon_0 +b_0)$. For $j \in \mathcal{S}^{\epsilon_0}$ moving to $j \in \mathcal{F}^{\epsilon_0 + b_0}$, $\mu_j^{\epsilon_0 + b_0} = \mu_j^{\epsilon_0} + b_0 r_j = 0$; if moving into $\mathcal{E}^{\epsilon_0 + b_0}$ , $\mu_j^{\epsilon_0 + b_0} = \mu_j^{\epsilon_0} + b_0 r_j = C$. Thus the $\mu_j$ paths are continuous in $\Delta \epsilon$ and between breakpoints (equivalently, in stable regions), they are either constant or linear in $\Delta \epsilon$ and as such are piecewise linear over perturbations $\Delta \epsilon$. 
\end{proof} 

\subsubsection{Proof of Proposition \ref{compare-utility}}
 \begin{proof} 
Fix $\epsilon \in (0,1)$ and consider the stable region of $\Delta \epsilon$ perturbations given by $(b_L, b_U)$. Suppose $b_L = \frac{g_j}{d_j}$ with $j\in \mathcal{E}$, then if $y_j = -1$, $\hat{y_j}  = +1$. Thus at the breakpoint $\Delta \epsilon = b_L$, $j$ moves into $\mathcal{S}^{^\epsilon + b_L}$ and $\hat{y_j} = +1$ and $u_{z_j}(\epsilon +b_L) < u_{z_j}(\epsilon)$ where $z_j$ gives the group membership of $\bx_j$. Since no other points transition,  $u_{\bar{z}}(\epsilon +b_L) = u_{\bar{z}}(\epsilon)$ for all $\bar{z} \neq z_j$. Since $b_L < 0$, the fairness constraint is tightened and associated with a shadow price given by $\gamma > 0$ such that $p(\epsilon + b_L) < p(\epsilon)$. \checkmark 
  
Suppose $b_L = \frac{C-\mu_j^\epsilon}{r_j}$ and $j \in \mathcal{S}^\epsilon$ with $y_j = +1$, then $j$ moves into $j \in \mathcal{E}^{\epsilon + b_L}$ such that $\hat{y_j} = -1$. Thus $u_{z_j}(\epsilon +b_L) < u_{z_j}(\epsilon)$ and $u_{\bar{z}}(\epsilon +b_L) = u_{\bar{z}}(\epsilon)$ where $z_j$ is the group membership of $\bx_j$ and $\bar{z} \neq z_j$, and $p(\epsilon + b_L) \le p(\epsilon)$. \checkmark

Suppose $b_U = \frac{g_j}{d_j} > 0$ where $j \in \mathcal{E}^\epsilon$, $y_j = +1$, and $\hat{y_j} = -1$. At the breakpoint, $j$ moves into $\mathcal{S}^{\epsilon + b_U}$ such that $y_j = -1$. Then $u_{{z_j}}(\epsilon +b_U) > u_{{z_j}}(\epsilon)$ where ${z_j}$ is the group membership of $\bx_j$. For $\bar{z} \neq z_j$, $u_{{\bar{z}}}(\epsilon +b_U) = u_{{\bar{z}}}(\epsilon)$, and since $b_U > 0$, the fairness constraint is loosened and $p(\epsilon+b_U) > p(\epsilon)$. 

Suppose $b_U = \frac{C-\mu_j^\epsilon}{r_j} > 0$ where $j \in \mathcal{S}^\epsilon$ and $y_j = -1$. At the breakpoint, $j$ moves into $\mathcal{E}^{\epsilon + b_U}$ such that $\hat{y_j} = +1$. Then $u_{{z_j}}(\epsilon +b_U) > u_{{z_j}}(\epsilon)$ where ${z_j}$ gives the group membership of $\bx_j$. For $\bar{z} \neq z_j$, $u_{{\bar{z}}}(\epsilon +b_U) = u_{{\bar{z}}}(\epsilon)$, and since $b_U > 0$, the fairness constraint is loosened and $p(\epsilon+b_U) \ge p(\epsilon)$. \checkmark
 \end{proof}
 
\subsubsection{Proof of Theorem \ref{pareto-thm}}
\begin{proof}
Theorem \ref{pareto-thm} follows from Proposition \ref{stable-perturbations}, Corollary \ref{corollary-stable}, Lemma \ref{lemma-piecewise}, and Proposition \ref{compare-utility}. 
\end{proof}

 \end{document}